\documentclass[onefignum,onetabnum]{siamart190516}

\author{Manish Raghavan\thanks{Harvard University
  (\email{mraghavan@seas.harvard.edu}). Most of this research has been done while the author was a graduate student at Cornell and an intern at Microsoft Research NYC.}
  \and
  Aleksandrs Slivkins\thanks{Microsoft Research, New York
  (\email{slivkins@microsoft.com}).}
 \and
 Jennifer Wortman Vaughan\thanks{Microsoft Research, New York
 (\email{jenn@microsoft.com}).}
 \and
 Zhiwei Steven Wu\thanks{Carnegie Mellon University
 (\email{zstevenwu@cmu.edu}). Most of this research has been done while the author was a postdoc at Microsoft Research NYC.}
 }
\usepackage{cleveref}
\usepackage{bm}
\usepackage{color}
\usepackage{tikz}
\usepackage{amsmath,amssymb}
\usepackage{xspace,nicefrac}

\def\citet{\cite}
\def\citep{\cite}

\newcommand{\rbr}[1]{\left(\,#1\,\right)}
\newcommand{\sbr}[1]{\left[\,#1\,\right]}
\newcommand{\cbr}[1]{\left\{\,#1\,\right\}}

\newcommand{\EqComment}[1]{\text{\emph{(#1)}}}
\newcommand{\OMIT}[1]{}
\newcommand{\ie}{{i.e.,~\xspace}}
\newcommand{\eg}{{e.g.,~\xspace}}

\newcommand{\xhdr}[1]{\vspace{2mm} \noindent{\bf #1}}
\newcommand{\polylog}{\operatornamewithlimits{polylog}}
\newcommand{\LDOTS}{\, ,\ \ldots\ ,}     
\newcommand{\Cel}[1]{{\lceil {#1} \rceil}}

\newcommand{\eps}{\varepsilon}

\newcommand{\mE}{\mathcal{E}}
\newcommand{\mF}{\mathcal{F}}
\newcommand{\mH}{\mathcal{H}}
\newcommand{\mN}{\mathcal{N}}

\newcommand{\initOneLiners}{%
 	\setlength{\itemsep}{0pt}
	\setlength{\parsep }{0pt}
  	\setlength{\topsep }{0pt}     	
}
\newenvironment{OneLiners}[1][\ensuremath{\bullet}]
    {\begin{list}
        {#1}
        {\initOneLiners}}
    {\end{list}}

\newenvironment{proofof}[1][]
	{\begin{proof}[Proof #1]}
    {\end{proof}}

\newcommand{\term}[1]{\ensuremath{\mathtt{#1}}\xspace}

\newcommand{\ALG}{\term{ALG}}

\newcommand{\perturb}{\eps^{\term{base}}}  
\newcommand{\Ebase}[1]{\mathbb{E}_{\term{base}}\sbr{ #1 }}
\newcommand{\Pbase}[1]{\mathbb{P}_{\term{base}}\sbr{ #1 }}
\newcommand{\KL}{\text{KL}}
\newcommand{\niceE}{\mE_{\term{nice}}} 

\newcommand{\simF}{\term{sim}} 

\newcommand{\con}{\term{con}} 
\newcommand{\CON}{\term{CON}} 

\newcommand{\PReg}{\term{PReg}} 
\newcommand{\rPReg}[1][{\rounds}]{\PReg^{#1}}  
\newcommand{\rReg}[1][{\rounds}]{R^{#1}} 

\newcommand{\bpreg}[1]{\PReg^{\rounds}(#1)}
\newcommand{\bpregi}[1]{\PReg^{#1}}

\newcommand{\basereg}[1]{R_0^{\rounds}\left(#1\right)}

\newcommand{\reg}[1]{R(#1)}
\newcommand{\regi}[1]{R^{#1}}

\newcommand{\BayesGreedy}{BatchBayesGreedy\xspace}
\newcommand{\FreqGreedy}{BatchFreqGreedy\xspace}
\newcommand{\GreedyStyle}{batch-greedy-style\xspace}

\newcommand{\Rew}[1][\theta]{\mathtt{Rew}_{#1}} 

\def\rounds{\mc T}

\newtheorem{remark}[theorem]{Remark}
\newtheorem{fact}[theorem]{Fact}

\def\R{\mathbb{R}}
\def\tran{^\top}

\newcommand{\p}[1]{\left(#1\right)}
\newcommand{\kl}[2]{\ensuremath{\text{KL}(#1 \, ||\, #2)}}
 \def\given{\;|\;}

\renewcommand{\b}[1]{\left[#1\right]}
\newcommand{\E}[1]{\mathbb{E}\b{#1}}

\newcommand{\D}{D}

\DeclareMathOperator*{\argmax}{arg\,max}
\DeclareMathOperator*{\argmin}{arg\,min}
\DeclareMathOperator{\Exp}{\mathbb{E}}
\def\R{\mathbb{R}}
\def\ty{{\lfloor t/Y \rfloor}}

\newcommand{\mc}[1]{\mathcal{#1}}
\def\tran{^{\top}}
\def\dx{\;dx}
\def\given{\;|\;}

\def\pmt{\overline \theta}
\def\pvt{\Sigma}
\def\bmt{\theta_t^{\textrm{bay}}}

\def\ab{a_t'}

\newcommand{\fmt}[1][t]{\theta_{#1}^{\textrm{fre}}}

\def\af{a_t}
\def\vrt{r_B}
\def\vrto{\mathbf{r}_{1:t_0-1}}
\def\Xto{X_{t_0-1}}
\def\Zto{Z_{t_0-1}}
\def\weights{w_B}
\def\prior{\mc P}
\def\thetahatt{\hat{\theta}_t}

\newcommand{\creg}[2]{\text{Regret}^{#1}(#2)}

\newcommand{\iR}[1]{R_{#1}} 

\def\hcat{\hat{c}_{a,t}}

\def\ZB{Z_B}
\def\WB{W_B}

\def\bg{\BayesGreedy}

\def\fg{\FreqGreedy}

\def\elt{\ensuremath{E_{\ell, t}}}

\begin{document} \title{Greedy Algorithm almost Dominates\\ in Smoothed Contextual Bandits%
\thanks{A version of our results have been announced in an extended abstract~\cite{externalities-colt18}, and fleshed out in the technical report~\citep{externalities-colt18-arxiv}. This paper is streamlined compared to \cite{externalities-colt18,externalities-colt18-arxiv}, focusing on the greedy algorithm. It has been available on {\tt arxiv.org/abs/2005.10624} since May 2020. The current version (since Dec'21) allows the random perturbations of context vectors to be correlated across actions, and adds a lower bound result.}}

\maketitle

\begin{abstract}
  Online learning algorithms, widely used to power search and content optimization on the web, must balance exploration and exploitation, potentially sacrificing the experience of current users in order to gain information that will lead to better decisions in the future. While necessary in the worst case, explicit exploration has a number of disadvantages compared to the greedy algorithm that always ``exploits'' by choosing an action that currently looks optimal. We ask under what conditions inherent diversity in the data makes explicit exploration unnecessary. We build on a recent line of work on the smoothed analysis of the greedy algorithm in the linear contextual bandits model. We improve on prior results to show that the greedy algorithm almost matches the best possible Bayesian regret rate of any other algorithm on the same problem instance whenever the diversity conditions hold. The key technical finding is that data collected by the greedy algorithm suffices to simulate a run of any other algorithm.  Further, we prove that Bayesian regret of the greedy algorithm is at most
$\tilde{O}(T^{1/3})$ in the worst case, where $T$ is the time horizon.

\end{abstract}


\section{Introduction}
\label{sec:intro}
Online learning algorithms are a key tool in web search and content optimization, adaptively learning what users want to see. In a typical application, each time a user arrives, the algorithm chooses among various content presentation options (\eg news articles to display), the chosen content is presented to the user, and an outcome (\eg a click) is observed. Such algorithms must balance \emph{exploration} (making potentially suboptimal decisions now for the sake of acquiring information that will improve decisions in the future) and \emph{exploitation} (using information collected in the past to make better decisions now). Exploration could degrade the experience of a current user, but improves user experience in the long run. This exploration-exploitation tradeoff is commonly studied in the online learning framework of \emph{multi-armed bandits}~\citep{Bubeck-survey12,slivkins-MABbook,LS19bandit-book}.

\normalmarginpar
Exploration is widely used, both in theory and in practice.
Yet, it has several important disadvantages.
%
First, exploration is wasteful and risky in the short term. It is undesirable for the current user, as something imposed only for the sake of the future users. Exploration may appear unfair, and may even be unethical or illegal in sensitive application domains such as medical decisions.
Second, exploration adds a layer of complexity to algorithm design (\eg see \citep{Langford-nips07,monster-icml14}), and its adoption at scale tends to require substantial systems support and buy-in from management \citep{MWT-WhitePaper-2016,DS-arxiv}. A system that only exploits would typically be cheaper to design and deploy.
Third, exploration runs into incentive issues when actions (\eg which product to buy) are controlled by users. In applications such as recommender systems, an algorithm can only encourage exploration via  recommendations and other provided information, but users would be reluctant to follow if it is not in their self-interest.%
\footnote{Making exploration compatible with users' incentives is possible, at least in theory, albeit costly; see \cite{IncentivizedExploration-chapter} for an overview of related research.}



\reversemarginpar
An algorithm without explicit exploration, a.k.a. the \emph{greedy algorithm}, always chooses the action that appears optimal according to current estimates of the problem parameters. Further, the greedy algorithm describes self-interested behavior of users in a recommendation system.%
\footnote{The formal model is as follows: users sequentially choose among available actions, after fully observing what happened with the previous users.}
The greedy algorithm it is known to perform poorly in a wide range of problem instances, yet it works remarkably well in some examples. A more detailed characterization for whether and when the greedy algorithm performs well --- put differently, whether and when exploration is not at all helpful --- is an important concern in the study of the exploration-exploitation tradeoff.



A recent line of work \citet{bastani2017exploiting,kannan2018smoothed}
analyzes conditions under which inherent diversity in the data makes explicit exploration unnecessary. They consider the \emph{linear contextual bandits}~\citep{Langford-www10,Reyzin-aistats11-linear,Csaba-nips11}, a standard variant of multi-armed bandits appropriate for content personalization scenarios. In particular, Kannan et al.~\cite{kannan2018smoothed} model data diversity via small perturbations of the context vectors, and focus on regret in expectation over these perturbations. They prove that the greedy algorithm achieves expected regret which scales as $\tilde{O}(\sqrt{T})$ in terms of the time horizon $T$.
This is the best regret rate that can be achieved in the worst case (\ie for all problem instances), even  without data diversity assumptions. However, this result does not resolve how the greedy algorithm compares to other algorithms under data diversity conditions, neither in the worst case nor for particular problem instances.



We expand on this line of work. We prove that under similar diversity conditions, the greedy algorithm almost matches the best possible Bayesian regret of \emph{any} algorithm \emph{on the same problem instance}. Known upper bounds on algorithms' Bayesian regret range from $\polylog(T)$ for some problem instances to $\tilde{O}(\sqrt{T})$ in the worst case, and each of them carries over to the greedy algorithm. Moreover, we prove that the Bayesian regret of the greedy algorithm scales as $\tilde{O}(T^{1/3})$ in the worst case, as long as there are at most $\polylog(T)$ feasible actions in each round.

The data diversity conditions in \citet{kannan2018smoothed} and this paper are inspired by the smoothed analysis framework of Spielman and Teng~\citet{SmoothedAnalysis-jacm04}, who proved that the expected running time of the simplex algorithm is polynomial for perturbations of any initial problem instance (whereas the worst-case running time has long been known to be exponential). Such disparity implies that very bad problem instances are brittle. We find a similar disparity for the greedy algorithm in our setting.

\subsection*{Our contributions}

We consider a Bayesian version of linear contextual bandits in which the latent weight vector $\theta$ is drawn from a known prior. In each round, an algorithm is presented several actions to choose from, each represented by a \emph{context vector}. The expected reward of an action is a linear product of $\theta$ and the corresponding context vector. The tuple of context vectors is drawn independently from a fixed distribution. In the spirit of smoothed analysis, we assume that this distribution has a small amount of jitter. Formally, in each round $t$ the tuple of context vectors is drawn from some fixed distribution, and then a small \emph{perturbation} $\eps_{a,t}$ is added to the context vector for each action $a$. The basic version adopted in \citet{kannan2018smoothed} is that each $\eps_{a,t}$ is an independent spherical Gaussian distribution; we call it the \emph{action-independent perturbation}. We allow a more general perturbation model, spelled out in Section~\ref{sec:model}, which can be correlated across actions, but independent across rounds and coordinates.
We are interested in Bayesian regret, i.e., regret in expectation over the Bayesian prior. Following the literature, we are primarily interested in the dependence on the time horizon $T$.


We focus on a batched version of the greedy algorithm, in which new data arrives to the algorithm's optimization routine in small batches, rather than every round. This property is essential for our analysis, and easy to implement in practice. As a restriction, it is well-motivated from a practical perspective: in high-volume applications data usually arrives to the ``learner" only after a substantial delay \citep{MWT-WhitePaper-2016,DS-arxiv}.

Our main result is that the greedy algorithm  matches the Bayesian regret \emph{of any algorithm} up to $\polylog(T)$ factors \emph{for each problem instance}, \ie fixing the Bayesian prior and the context distribution.
This holds for two natural versions of the batched greedy algorithm, Bayesian and frequentist, henceforth called \BayesGreedy and \FreqGreedy. For \BayesGreedy, the chosen action maximizes expected reward according to the Bayesian posterior. The regret bound holds for any Bayesian prior. \FreqGreedy estimates $\theta$ using ordinary least squares regression and chooses the best action according to this estimate. The regret bound and comes with an extra additive polylogarithmic factor, but is stronger in that the algorithm does not need to know the prior. This result requires a Gaussian prior, which can, however, be very concentrated.

The key insight is that the data collected with perturbed contexts can be used to simulate a run of any other algorithm $\ALG$, with the number of rounds scaled down by some factor $Y$. (This simulation arises only as a technique in the analysis.) It follows that \BayesGreedy at each round $t$ knows at least as much as $\ALG$ after $t/Y$ rounds, so  its selection is at least as good as that of $\ALG$. To handle the frequentist algorithm, we consider a hypothetical algorithm that receives the same data as \FreqGreedy, but chooses actions like \BayesGreedy. We analyze this hypothetical algorithm using a similar `simulation argument', and then upper-bound the difference in Bayesian regret compared to \FreqGreedy.


Next, we argue that our problem remains difficult despite perturbations. Specifically, we prove that any algorithm achieves Bayesian regret no better than $\tilde{\Omega}(\sqrt{T})$ in the worst case. This holds even if there are at most two feasible actions in each round, and even if perturbation size can be an absolute constant. For this  lower bound, perturbations on both actions are completely correlated (\ie identical).

Finally, we consider action-independent perturbation, and analyze Bayesian regret in the worst case over all Bayesian instances. We prove that LinUCB algorithm~\citep{Langford-www10,Reyzin-aistats11-linear,Csaba-nips11},
a standard algorithm for linear contextual bandits, achieves Bayesian regret $\tilde{O}(K^{2/3}\,T^{1/3})$ if there are at most $K$ feasible actions in each round. Consequently, a similar regret bound holds for \BayesGreedy and \FreqGreedy. The $\tilde{O}(T^{1/3})$ regret rate is a mathematical curiosity, as we are not aware of any published regret bounds between $\sqrt{T}$ and $\polylog(T)$; however, it is unclear if this regret rate is optimal. Regardless, we conclude that action-independent perturbation is substantially ``easier" compared to the general case, in light of the lower bound stated above.


\subsection*{Map of the paper}

The paper continues with related work (Section~\ref{sec:related-work}), model and preliminaries (Section~\ref{sec:model}), precise statements of the results (Sections~\ref{sec:bayesian_greedy}), and a detailed discussion of the techniques (Section~\ref{sec:bayesian_greedy-key}). The analysis is spelled out in Sections~\ref{app:pf_bg}-\ref{app:linucb}, for, resp., the greedy algorithms, the lower bound, and  LinUCB algorithm. Some tools are moved to the appendix so as not to interrupt the flow.

\section{Related Work}
\label{sec:related-work}

The greedy algorithm works well in some examples, and badly in some others. This has been a folklore knowledge for decades, and it has been confirmed in extensive recent experiments \citet{practicalCB-arxiv18}. One way to formalize a \emph{negative} result is to consider a Bayesian prior over 2-armed bandit instances. Then, with positive-constant probability over the prior, the greedy algorithm fails to explore the best arm, and therefore incurs a positive-constant regret in each round (see Chapter 11 in \cite{slivkins-MABbook}). This is a very general result, as it holds for any Bayesian prior.

\subsection*{Positive results on the greedy algorithm}
Most related to ours are papers by Kannan et al.~\cite{kannan2018smoothed} and Bastani et al.~\cite{bastani2017exploiting}.%
\footnote{An early version of Bastani et al.~\citet{bastani2017exploiting} (v2, Jun'17) is prior work relative to this paper. In particular, it focuses on the special case of two actions. Subsequent versions are concurrent work.}
Both study the greedy algorithm in linear contextual bandits with data diversity conditions. In particular, Kannan et al.~\cite{kannan2018smoothed} introduce the perturbation model adopted in our paper, focusing on the special case of action-independent perturbations. We provide a detailed comparison below.

We substantially improve over the $\tilde{O}(\sqrt{T})$ regret bound from Kannan et al.~\citet{kannan2018smoothed}: our main result applies per-instance rather than only in the worst-case, and allows perturbations to be correlated across actions. Going back to action-independent perturbations, as in \citet{kannan2018smoothed}, we also improve the worst-case bound on Bayesian regret to $\tilde{O}(T^{1/3})$ when there are only $\polylog(T)$ feasible actions in each round. However, these improvements come at the cost of some additional assumptions. First, we consider Bayesian regret, whereas their regret bound holds for each realization of $\theta$.
Second, they allow the context vectors to be chosen by an adversary before the perturbation is applied. Third, they extend their analysis to a somewhat more general model, in which there is a separate latent weight vector for every action (which amounts to a different model of perturbations). However, this extension relies on the greedy algorithm being initialized with a substantial amount of data.

Bastani et al.~\cite{bastani2017exploiting} show that the greedy algorithm achieves logarithmic regret in a version of linear contextual bandits that is incomparable to ours in several important ways. First, the actions share a common context
vector in each round, but are parameterized by different latent vectors. Then, playing a given arm reveals no information about the other arms, which makes their problem more difficult compared to ours.
To compensate for this difficulty, they posit a strong assumption on data diversity: essentially, that the distribution of contexts is approximately symmetric around the origin. It follows that for any pair of arms, each arm is better than the other for a constant
fraction of rounds. In contrast, our model allows the context distribution to be arbitrary, subject to a relatively small perturbation; in particular, the same action could be the best action in all rounds.
Third, a version of Tsybakov's \emph{margin condition} is assumed, which is
known to substantially reduce regret rates in bandit problems (see, \eg
\cite{Zeevi-colt10}).
Instead, we assume Gaussian perturbations, allowing us to make a finer-grained simulation argument that the
greedy algorithm is instance-optimal.



Acemoglu et al.~\cite{AcemogluMMO19} and Immorlica et al.~\cite{Jieming-unbiased18}%
\footnote{The early version of Acemoglu et al.~\cite{AcemogluMMO19} (from Nov'17) is prior work relative to this paper; subsequent versions are conrurrent work. Immorlica et al.~\cite{Jieming-unbiased18} is subsequent work.}
analyze the greedy algorithm from the economics perspective, providing positive results for ``greedy" self-interested behavior of users in a recommendation system. Acemoglu et al.~\cite{AcemogluMMO19} study heterogenous users with private types. In our terms, it is a version of contextual bandits in which the current context is not observed in the future rounds. Among other results, they prove that the greedy algorithm works well in this setting, under strong heterogeneity assumptions incomparable with yours. In particular, they postulate that a user arriving in each round inherently prefers each arm with (at least) a constant probability. Immorlica et al.~\cite{Jieming-unbiased18} constructs a data disclosure policy which reveals to each user the history for a predetermined subset of prior users. In our terms, they consider a bandit problem with constantly many arms, and a greedy algorithm operating on limited data as prescribed by this policy. They prove that such algorithm attains regret rates that are near-optimal for any bandit algorithm.

\subsection*{Technical aspects}

Any contextual bandit algorithm can be simulated using data collected by any other contextual bandit algorithm which independently randomizes actions in each round
\cite{Langford-wsdm11,Dudik-uai12}. Essentially, the required number of samples is inversely proportional to the smallest sampling probability across arms. While not very complicated technically, this approach works for contextual bandits (linear or not) without any additional assumptions. However, this approach fails in our setting because our data is collected by a \emph{deterministic} algorithm. Instead, our simulation uses a different approach, which relies on random perturbation of contexts.

The work on ``batched bandit algorithms" \cite{Perchet2015BatchedBP} assumes that the rounds are partitioned into ``batches" so that the algorithm cannot use the data from the current batch. The goal is to achieve efficient exploration despite this restriction. In contrast, we focus on the greedy algorithm rather than exploration, and invoke the batch property as a voluntary feature which helps in the analysis.

\subsection*{Linear contextual bandits}

The problem was introduced in \citet{Langford-www10}, motivated by personalized news recommendations. The non-contextual version stems from \citet{Auer-focs00}. Both versions have been studied extensively, see books \cite{Bubeck-survey12,LS19bandit-book} for background.

Algorithm LinUCB, which we discuss in Section~\ref{app:linucb}, implements `optimism under uncertainty', a common paradigm for problems with explore-exploit tradeoff. The algorithm was defined in \citet{Langford-www10}, and analyzed in \citep{Reyzin-aistats11-linear,Csaba-nips11}. (A non-contextual version of LinUCB was introduced earlier in \citet{Auer-focs00}, and analyzed in \cite{DaniHK-colt08}.)
The details of the algorithm differ subtly between the papers; we focus on the version from \citet{Csaba-nips11}.

LinUCB achieves regret
    $\tilde{O}(d\sqrt{T})$,
where $d$ is the dimension, for any number of actions \cite{Csaba-nips11}. Any algorithm suffers regret
    $\Omega(d\sqrt{T})$
in the worst case \cite{DaniHK-colt08}. LinUCB has been observed to perform well even when the problems are not linear \citep{semi-CB-nips16}.

\section{Our Model and Preliminaries}
\label{sec:model}

\normalmarginpar
We consider the model of \emph{linear contextual
  bandits}~\citep{Langford-www10,Reyzin-aistats11-linear,Csaba-nips11}. A learner operates over $T$ timesteps (a.k.a. rounds),
where $T$ is a known time horizon.%
\footnote{For intuition, each round typically corresponds to an interaction with a new user.}
Each round $t$ proceeds as follows. There are at most $K$ actions available, a.k.a. \emph{arms}. Denote the action set as
    $A_t \subset \cbr{ 1 \LDOTS K}$.
Each action $a\in A_t$ is associated with a \emph{context vector} $x_{a,t} \in \R^d$, which may contain features of the action and/or the round.
We assume that the tuple of context vectors
    $\rbr{ x_{a,t}: a\in A_t }$
is drawn independently from a fixed distribution $\D$.  The learner observes this tuple, selects an action $a_t\in A_t$, and observes reward $r_t$.
We assume that $r_t$ is drawn independently from some distribution determined by the chosen context vector $x_{a_t,t}$, and the expected reward is linear in this vector. More
precisely, we let $r_{a,t}$ be the reward of each action $a\in A_t$ if this action
is chosen in round $t$ (so that $r_t = r_{a_t,t}$), and posit an unknown vector $\theta\in\R^d$ such that
\[ \E{r_{a, t}\mid x_{a, t}} = \theta\tran x_{a, t} \quad\text{for any round $t$ and action $a\in A_t$}.\]
So far, it is a standard \emph{frequentist} formulation of linear contextual bandits. It is determined by time horizon $T$, dimension $d$, number $K$ of feasible actions per round, context distribution $\D$, and the latent vector $\theta$.

We consider a natural \emph{Bayesian} version, where $\theta$ drawn from a known Bayesian prior $\prior$. Thus, a problem instance consists of parameters $T,d,K$, context distribution $\D$, and Bayesian prior $\prior$. The prior can be arbitrary unless specified otherwise.

\OMIT{Throughout most of the paper, the realized rewards are
either in $\{0,1\}$ or are the expectation plus independent Gaussian
noise of variance at most $1$.
\mr{Are rewards ever $\{0, 1\}$ in this version?}}


The learner strives to maximize the expected total reward over $T$ rounds, or
    $\sum_{t=1}^T \E{\theta\tran x_{a, t}}$.
We focus on \emph{regret}, a standard performance measure which compares the learner to the all-knowing benchmark: a hypothetical algorithm that knows the best action in each round. Formally, we define the best context vector in round $t$ as
\[ x^*_t \in \argmax_{x \in \cbr{ x_{a, t}:\; a\in A_t}} \theta\tran x, \]
\ie a context vector which achieves the highest expected reward. Next,
\begin{align}\label{eq:regret-def}
\text{Regret}(T) = \textstyle
    \sum_{t=1}^T \theta\tran x_t^* -
\theta\tran x_{a_t, t}.
\end{align}
\emph{Expected regret} is defined as the expectation of \eqref{eq:regret-def} over the
context vectors, the rewards, and the algorithm's random seed. We are mainly interested in
\emph{Bayesian regret}, where the expectation is taken over all of the above and
the prior over $\theta$.

\subsection*{Data diversity}
We model data diversity via the following process, called \emph{perturbed context generation}. Fix round $t$, and recall that $A_t$ denotes the set of available actions. First, a tuple
    $\rbr{ \mu_{a,t}\in \R^d:\, a\in A_t}$
of \emph{mean context vectors} is drawn independently from some fixed distribution $\D_\mu$ over $(\R^d)^{|A_t|}$.
Then for each action $a\in A_t$, the context vector is
    $x_{a,t} = \mu_{a,t} + \eps_{a, t}$,
where $\eps_{a, t}\in \R^d$ is a zero-mean perturbation vector. Marginally, each perturbation  vector $\eps_{a, t}$ is distributed as $\mN(0,\rho_{a,t}\cdot I)$, a spherical Gaussian distribution over $\R^d$ with zero mean and per-coordinate standard deviation $\rho_{a,t}>0$. We consider two basic versions for correlation across actions:
\begin{itemize}
\item \emph{action-independent perturbation}:  each perturbation vector $\eps_{a,t}$ is an independent draw from $\mN(0,\rho I)$.

\item \emph{fully-action-correlated perturbation}: $\eps_{a,t} = \eps_t$ for all arms $a\in A_t$, where the (common) perturbation vector $\eps_t$ is an independent draw from $\mN(0,\rho I)$.
\end{itemize}
Our guarantees deteriorate if \emph{perturbation size} $\rho$ is very small.

We allow a more general model of action-correlation which interpolates between these two extremes defined above. We have a set covering $\mF_t$ of $A_t$, \ie a family $\mF_t$ of subsets of $A_t$ whose union equals $A_t$. For each subset $S\in \mF_t$, we have a \emph{base perturbation} $\perturb_{S,t}\in \R^d$, which is an independent draw from $\mN(0,\rho_{S,t}\cdot I)$, for some $\rho_{S,t}>0$. We sum up the base perturbations over all relevant subsets $S\in\mF_t$:
\begin{align}\label{eq:correlated-perturbation}
     \eps_{a,t} = \sum_{S\in \mF_t:\; a\in S} \perturb_{S,t} .
\end{align}
The paradigmatic case is that $\mF_t=\mF$ for all rounds $t$, but we allow it to change over time. Likewise, the paradigmatic case is that $\rho_{a,t} = \rho$ for all arms $a$ and rounds $t$, but we allow $\rho_{S,t}$ can vary for different subsets $S$ and rounds $t$. In the latter case, we summarize the dependence on the perturbations via the \emph{perturbation size}
    \[ \rho := \min_{t\in [T],\, a\in A_t}\; \max_{S\in \mF_t:\, a\in S} \rho_{S,t}. \]
Note that we ``use" the largest relevant perturbation for a given arm-round pair.



\OMIT{ 
\mredit{We will assume that $\eps_{a_1, t}$ and $\eps_{a_2, t}$ may be
correlated for $a_1 \ne a_2$, which leads to two natural special cases:
\begin{itemize}
  \item $\eps_{a_1, t}$ and $\eps_{a_2, t}$ are independent for
    all $a_1 \ne a_2$, i.e., each perturbation is an independent random Gaussian
    vector.
  \item $\eps_{a_1, t} = \eps_{a_2, t}$ for all $a_1, a_2$, i.e.,
    all contexts within a timestep $t$ are perturbed by the same random noise
    $\eps_t$.
\end{itemize}
}} 


We make several technical assumptions. First, the distribution $\D_\mu$ is such that each context vector has bounded $2$-norm, i.e., $\|\mu_{a,t}\|_2 \le 1$. It can be arbitrary otherwise. Second, the perturbation size needs to be sufficiently small compared to the dimension $d$,  $\rho\leq 1/\sqrt{d}$.
Third, the realized reward $r_{a,t}$ for each action $a$ and round $t$ is
    $r_{a,t} = x_{a,t}\tran \theta + \eta_{a,t}$,
the mean reward $x_{a,t}\tran \theta$ plus standard Gaussian noise $\eta_{a,t}\sim \mc N(0, 1)$.\footnote{Our analysis can be easily extended to handle reward noise of fixed variance, i.e.,
    $\eta_{a,t}\sim \mc N(0, \sigma^2)$.
\fg would not need to know $\sigma$. \bg would need to know either $\Sigma$ and $\sigma$ or just $\Sigma/\ \sigma^2$.}

\subsection*{Batched greedy algorithms}
We write $x_t$ for $x_{a_t,t}$, the context vector chosen at time $t$. The history up to round $t$ is the tuple $h_t = ((x_{1}, r_1) \LDOTS (x_t, r_t))$.

For the batch version of the greedy algorithm, time is divided in batches of $Y$ consecutive rounds each. When forming its estimate of the optimal action at round $t$, the algorithm may only use the history up to the last round of the previous batch, denoted $t_0$. We consider both Bayesian and frequentist versions, called \bg and \fg.

\bg forms a posterior over $\theta$ using prior $\prior$ and history $h_{t_0}$.  In round $t$ it chooses the action that maximizes reward in expectation over this posterior. This is equivalent to choosing
\begin{align}\label{eq:BG-est-defn}
 a_t = \argmax_a  x_{a,t}\tran \, \bmt, \quad
    \text{where $\bmt := \Exp[\theta \mid h_{t_0}] $ }.
\end{align}

\fg~does not rely on any knowledge of the prior. It chooses the best action according to the least squares estimate of $\theta$, denoted $\fmt$, computed with respect to history $h_{t_0}$:
\begin{align}\label{eq:FG-est-defn}
 \textstyle a_t = \argmax_a  x_{a, t}\tran \, \fmt, \quad
\text{where $\fmt = \argmin_{\theta'} \sum_{\tau = 1}^{t_0} ((\theta') \tran
x_{\tau} - r_{\tau})^2$}.
\end{align}

\subsection*{Empirical covariance matrix}
Fix round $t$. Let $X_t\in \R^{t\times d}$ be the \emph{context matrix}, a matrix whose rows are vectors $x_1 \LDOTS x_t\in \R^d $.
A $d\times d$ matrix
    \[ Z_t \textstyle  :=\sum_{\tau=1}^t x_\tau x_\tau\tran = X_t\tran X_t, \]
called the \emph{empirical covariance} matrix, is an important concept in some
of the prior work on linear contextual bandits (\eg
\cite{Csaba-nips11,kannan2018smoothed}), as well as in this paper.

\subsection*{A note on notation}  We adopt a common (albeit slightly non-standard) convention that $\tilde{O}(\cdot)$ hides $\polylog(T)$ factors, regardless of the expression in brackets. In particular, the expression in brackets is always interpreted as a function of $T$.

\section{Statement of the Results}
\label{sec:bayesian_greedy}

We prove that in expectation over the random perturbations, both greedy algorithms favorably compare to any other algorithm. For any specific problem instance, both algorithms match the Bayesian regret of any algorithm on that particular instance up to polylogarithmic factors. We  state  the theorem  in terms of the main relevant parameters $T$, $K$, $d$, $Y$, and $\rho$.

\begin{theorem}
With perturbed context generation, there is some
  $Y_0 = \polylog(d, T)/\rho^2$
such that with batch duration $Y\geq Y_0$, the following holds. Fix any bandit algorithm, and let $R_0(T)$ be its Bayesian regret on a particular problem instance. Then on that same instance,
\begin{itemize}
\item[(a)] \bg has Bayesian regret at most
    $Y \cdot R_0(T/Y) + \tilde O(1/T)$,
\item[(b)] Suppose prior $\prior$ is a multivariate Gaussian distribution with invertible covariance matrix $\pvt$, and the eigenvalues of $\pvt$ are at least $\rho^{-4}/T$. Then \fg has Bayesian regret at most
    $Y \cdot R_0(T/Y) + \tilde{O}(C_{\pvt}\,\sqrt{d}/\rho^2)$,
    where $C_{\pvt}$ is determined by $\pvt$.
\end{itemize}
\label{thm:main-greedy}
\end{theorem}

\begin{remark}\label{rem:cov-matrix}
The dependence on the covariance matrix $\pvt$ in Theorem~\ref{thm:main-greedy}(b) is
        $C_{\pvt}= \sqrt{\lambda_{\max}} + 1/\sqrt{\lambda_{\min}}$,
where $\lambda_{\max}$, $\lambda_{\max}$ are, resp., the largest and smallest eigenvalues of $\pvt$. (This comes from Theorem~\ref{thm:bg_fg}.) The dependence on $\lambda_{\min}$ captures the deterioration in Bayesian regret if the prior is very concentrated. For example, if prior $\prior$ is independent over the components of $\theta$, with variance $\sigma^2\leq 1$ in each component, then
    $\lambda_{\max}\leq 1$ and $\lambda_{\min}=\sigma$,
so that
    $C_{\pvt}= 1 + 1/\sqrt{\sigma}$.
\end{remark}


Next, we prove that $\tilde{\Omega}(\sqrt{T})$ lower bound on Bayesian regret holds even under perturbed context generation. We posit the most difficult regime: $K=d=2$ and constant $\rho$. The lower bound focuses on fully-action-correlated perturbation.

\begin{theorem}\label{thm:LB}
Consider perturbed context generation with fully-action-correlated perturbation. Any algorithm achieves Bayesian regret no better than $\tilde{\Omega}(\sqrt{T})$ for some problem instance. This holds even if there are only $d=2$ dimensions, $K=2$ feasible actions in each round, and perturbation size $\rho$ is an absolute constant.
\end{theorem}

Finally, we focus on worst-case Bayesian regret. We consider action-independent perturbation, and posit that prior $\prior$ is a multivariate Gaussian with mean vector $\pmt$ and invertible covariance matrix $\pvt$.



\begin{theorem}
Consider perturbed context generation with action-independent perturbation.
Assume that all eigenvalues of the covariance matrix $\pvt$ are at most $1$,%
\footnote{In particular, if the prior $\prior$ is independent across the coordinates of
$\theta$, then the variance in each coordinate is at most $1$.}
and the mean vector satisfies
    $\|\pmt\|_2\geq 1+\sqrt{3 \log T} $.
Then
\begin{itemize}
\item[(a)] LinUCB algorithm, with appropriate parameter settings, has Bayesian regret
    $\term{reg} := \tilde O(d^2\,K^{2/3}/\rho^2)\times T^{1/3}$.
  \item[(b)] (Follows from Theorem~\ref{thm:main-greedy}.) Let $Y_0$ from Theorem~\ref{thm:main-greedy} be the batch duration. Then \bg has Bayesian regret at most $\tilde O(\term{reg})$. Likewise, \fg has Bayesian regret at most
    $\tilde O\rbr{\term{reg} + \sqrt{d}\,\rho^{-2}/\sqrt{\lambda_{\min}(\pvt)}}$,
    provided that
    $\lambda_{\min}(\pvt)\geq \rho^{-4}/T$.
\end{itemize}
\label{thm:main-worst-case}
\end{theorem}

\begin{remark}
The assumption $\|\pmt\|_2\geq 1+\sqrt{3 \log T} $ in Theorem~\ref{thm:main-worst-case} can be replaced with an assumption that the dimension $d$ is sufficiently large: $d \ge \log T/\log \log T$.
\end{remark}

\OMIT{ 
\begin{remark}
The regret bound for \bg in Theorem~\ref{thm:main-worst-case} does not depend on the covariance matrix $\pvt$ (as long as $\lambda_{\max}(\pvt)\leq 1$). The regret bounds for $\fg$ has an additive factor
\asedit{    $\tilde{O}(C_{\pvt}\,\sqrt{d}/\rho^2)$,
as per Theorem~\ref{thm:main-greedy}(b).}
\end{remark}
\mrcomment{Need to be careful wherever we use $C_{\Sigma}$ because of the bug
fix.}
} 

\OMIT{ 
As a corollary, we derive an ``approximately prior-independent" regret bound for \fg. A prior-independent regret bound is one that holds for all point priors (and therefore for all priors). Such result would be significant for \fg because it does not require any knowledge of the prior.  We provide an approximate version which does not hold for point priors, but allows the prior to be very sharp. For clarity, we state this result for independent priors. Assuming that the prior $\prior$ is independent over the components of $\theta$, we observe that the regret bound in Theorem~\ref{thm:main-worst-case}(b) holds whenever $\prior$ has standard deviation on the order of (at least) $T^{-2/3}$ in each component.

\begin{corollary}
Assume that the prior $\prior$ is independent over the components of $\theta$, with standard deviation $\kappa_i \in [T^{-2/3},1]$ in each component $i$. Suppose the mean vector satisfies
    $\|\pmt\|_2\geq 1+\sqrt{3 \log T} $.
With perturbed context generation, if $Y\geq Y_0$ as in Theorem~\ref{thm:main-greedy}, then \fg has Bayesian regret at most
    $\tilde O(d^2\,K^{2/3}\;T^{1/3}/\rho^2)$.
\label{cor:sharp-priors}
\end{corollary}
} 

\section{Overview: Key Techniques}
\label{sec:bayesian_greedy-key}

The key idea is to show that, with perturbed context generation, \BayesGreedy collects data that is informative enough to
``simulate'' the history of contexts and rewards from the run of any other algorithm \ALG over fewer rounds. This implies that it remains competitive with \ALG since it has at least as much information and makes myopically optimal decisions.

Let us formulate what we mean by ``simulation". We want to use the data collected from a single batch $B$ in order to simulate the reward for any one context $x$, and we want to accomplish this without knowing the latent vector $\theta$. More formally, we use the tuple
     $h_B = ((x_t,\,r_t):\; t\in B)$,
which we call the \emph{batch history}.
We are interested in the randomized function $\Rew(\cdot)$ that takes a context vector $x$ and outputs an independent random sample from $\mathcal{N}(\theta\tran x, 1)$; this is the realized reward for an action with context vector $x$. So, the function $\Rew(\cdot)$ is what we want to simulate using batch history $h_B$. To do so, we construct a fixed function $g$ such that $g(x,h_B)$ is distributed identically to $\Rew(x)$, for any fixed context vector $x$; the randomness in $g(x,h_B)$ comes from $h_B$.

This definition needs to be refined, so as to simulate independent noise in rewards. Indeed, randomness in $h_B$ comes from several sources: context arrivals, algorithm's decisions, realization of $\theta$, and observed rewards. Relying on the first three sources introduces dangerous correlations. To rule them out, we require our simulation $g(x,h_B)$ to have the same distribution as $\Rew(x)$, even if we condition on the context vectors $x_t$ previously chosen by the algorithm during this batch, \ie on the tuple $(x_t:t\in B)$.

\begin{definition}
Consider batch $B$ in the execution of \bg. Batch history
$h_B$ can simulate $\Rew(\cdot)$ up to radius $R>0$ if there exists a function
    $g: \{\text{context vectors}\}\times \{ \text{batch histories $h_B$}\} \to \R$
such that $g(x,h_B)$ is distributed identically to $\Rew(x)$, conditional on
the tuple $(x_t: t\in B)$, for all $\theta$ and all context vectors $x\in \R^d$ with $\|x\|_2\leq R$.
\label{def:simulation}
\end{definition}

Let us comment on how it may be possible to simulate $\Rew(x)$. For intuition, suppose that
    $x = \tfrac12\, x_1 + \tfrac12\, x_2$.
Then $(\tfrac12\, r_1 + \tfrac12\, r_2 + \xi)$ is distributed
as $\mathcal{N}(\theta\tran x, 1)$ if $\xi$ is drawn independently
from $\mathcal{N}(0, \tfrac12)$. Thus, we can define
    $g(x,h) = \tfrac12\, r_1 + \tfrac12\, r_2 + \xi$
in Definition~\ref{def:simulation}. We generalize this idea and
show that a batch history can simulate $\Rew$
with high probability as long as the batch size $Y$ is sufficiently large.

\begin{lemma}
 With perturbed context generation,
  there is some $Y_0 = \polylog(d, T)/\rho^2$ and
  $R = O(\rho \sqrt{d\log(TKd)}) $ such that with probability at least
  $1-T^{-2}$ any batch history from \bg can
  simulate $\Rew(\cdot)$ up to radius $R$, as long as $Y\geq Y_0$.
    \label{lem:simulation}
\end{lemma}

To prove this, we ensure that the data collected in batch $B$ are sufficiently diverse. To define ``sufficiently diverse", let the \emph{batch context matrix}, denoted $X_B$, be a matrix which comprises the context vectors $(x_t: t\in B)$. Namely, $X_B$ is the $Y\times d$ matrix whose rows are vectors $x_t$, $t\in B$, in the order of increasing $t$. Similarly to the ``empirical covariance matrix", we define the \emph{batch covariance matrix} as
\begin{align}\label{eq:ZB-defn}
  \ZB := \textstyle  X_B\tran\, X_B = \sum_{t\in B} x_t\, x_t\tran.
\end{align}
We think of data diversity in terms of the minimal eigenvalue of $\ZB$: the larger it is, the more diverse is the data. And we prove that the minimal eigenvalue of $\ZB$ is sufficiently large whenever $Y\geq Y_0$.

If the batch history of an algorithm can simulate $\Rew$, the algorithm has enough information to simulate the outcome of a fresh round of any other algorithm
$\ALG$. We use a coupling argument in which
we couple a run of \bg with a slowed-down run of $\ALG$, and prove
that the former accumulates at least as much information as the
latter, and therefore the Bayesian-greedy action choice is, in
expectation, at least as good as that of $\ALG$. This yields the regret bounds for \bg in Theorems~\ref{thm:main-greedy} and \ref{thm:main-worst-case}.

We use the same technique to handle \FreqGreedy. To treat both greedy algorithms at once, we define a template that unifies them. A bandit algorithm is called \emph{\GreedyStyle} if it divides the timeline in batches of Y consecutive rounds each, in each round $t$ chooses some estimate $\theta_t$ of $\theta$, based only on the data from the previous batches, and then chooses the best action according to this estimate, so that
   $a_t = \argmax_a \theta_t\tran x_{a,t}$.
Lemma~\ref{lem:simulation} extends to any \GreedyStyle algorithm.

The analysis of \FreqGreedy requires an additional step. We consider a hypothetical \GreedyStyle algorithm which separates data collection and reward collection: it receives feedback based on the actions of \fg, but collects rewards based on the (batched) Bayesian-greedy selection rule. We analyze this hypothetical algorithm using Lemma~\ref{lem:simulation}, and then argue that its Bayesian regret cannot be much smaller than that of \fg. Intuitively, this is because the two algorithms form
very similar estimates of $\theta$, differing only in the fact that the hypothetical algorithm uses the prior $\prior$ as well as the data. Due to this similarity, we show that the numerical difference between the two estimates at time $t$ is at most $\tilde{O}(1/t)$, even though either estimate is typically $\Omega(1/\sqrt{t})$ away from $\theta$.
This adds up to a maximal difference of $O(\log T)$ in Bayesian regret between the two algorithms, and completes our regret bounds for \fg.

\section{Analysis: Greedy Algorithms}
\label{app:pf_bg}
We present the proofs for all results on greedy algorithms. This section is structured as follows. In Section~\ref{app:pf_bg:diversity}, we quantify the diversity of data collected by \GreedyStyle algorithms, assuming perturbed context generation. In Section~\ref{app:pf_bg:simulation}, we show that a sufficiently ``diverse" batch history suffices to simulate the reward for any given context vector, in the sense of Definition~\ref{def:simulation}. Jointly, these two subsections imply that any batch history generated by a \GreedyStyle algorithm can simulate rewards with high probability, as long as the batch size is sufficiently large. Section~\ref{sec:bg-proofs-bg} builds on this foundation to derive regret bounds for \bg. The crux is that the history collected by \bg suffices to simulate a ``slowed-down" run of any other algorithm. This analysis extends to a version of \fg equipped with a Bayesian-greedy prediction rule (and tracks the performance of the prediction rule). Finally, Section~\ref{sec:bg-proofs-fg} derives the regret bounds for \fg, by comparing the prediction-rule version of \fg with \fg itself.

\xhdr{Preliminaries.}
We assume perturbed context generation in this section, without further mention. We use definitions for batch greedy algorithms from Section~\ref{sec:bayesian_greedy-key}: batch-greedy-style algorithm, batch history, batch context matrix, and batch covariance matrix. Throughout, we will use the following parameters as a shorthand:
\begin{align*}
\delta_R &= T^{-2} \\
\hat R  & = \rho\sqrt{2 \log(2 TKd/\delta_R)} \\
R   &= 1 + \hat{R} \sqrt{d}.
\end{align*}
Recall that $\rho$ denotes perturbation size, and $d$ is the dimension. The meaning of $\hat R$ and $R$ is that they are high-probability upper bounds on the perturbations and the contexts, respectively. More formally, by Lemma~\ref{lem:subgaussian_max} we have:
\begin{align}
\Pr\left[ \|\eps_{a,t}\|_\infty \leq \hat R:\;
    \text{ for all arms $a$ and all rounds $t$ } \right] &\geq 1-\delta_R
    \label{eq:bg-proofs-highprob-R-hat}\\
\Pr\left[ \|x_{a,t}\|_2 \leq R:\;
    \text{ for all arms $a$ and all rounds $t$ } \right] &\geq 1-\delta_R
    \label{eq:bg-proofs-highprob-R}
\end{align}

\OMIT{ 
Let us recap some of the key definitions from Section~\ref{sec:bayesian_greedy-key}. We consider \GreedyStyle algorithms, a template that unifies \BayesGreedy and \FreqGreedy. A bandit algorithm is called \emph{\GreedyStyle} if it divides the timeline in batches of Y consecutive rounds each, in each round $t$ chooses some estimate $\theta_t$ of $\theta$, based only on the data from the previous batches, and then chooses the best action according to this estimate, so that
 $a_t = \argmax_a \theta_t\tran x_{a,t}$.

For a batch $B$ that starts at round $t_0+1$, the \emph{batch history} $h_B$ is the tuple
    $((x_{t_0+\tau},\,r_{t_0+\tau}):\; \tau\in [Y])$,
and the \emph{batch context matrix} $X_B$ is the matrix whose rows are vectors
    $(x_{t_0+\tau}:\; \tau\in [Y])$.
Here and elsewhere, $[Y] = \{1, \cdots, Y\}$. The \emph{batch covariance matrix} is defined as
\begin{align}\label{eq:ZB-defn}
\ZB := X_B\tran\, X_B = \sum_{t=t_0+1}^{t_0+Y} x_t\, x_t\tran.
\end{align}
} 

\subsection{Data Diversity under Perturbations}
\label{app:pf_bg:diversity}

We are interested in the diversity of data collected by \GreedyStyle algorithms, assuming perturbed context generation. Informally, the observed contexts
    $x_1, x_2,\,\ldots $
should cover all directions in order to enable good estimation of the latent vector $\theta$. Following Kannan et al.~\cite{kannan2018smoothed}, we quantify data diversity via the minimal eigenvalue of the empirical covariance matrix $Z_t$. More precisely, we are interested in proving that $\lambda_{\min}(Z_t)$ is sufficiently large. We adapt some tools from \cite{kannan2018smoothed}, extending them from to action-correlated perturbations, and then derive some improvements for \GreedyStyle algorithms.

\subsubsection{Tools from~\citet{kannan2018smoothed}}

Kannan et al.~\cite{kannan2018smoothed} prove that for action-independent perturbations, $\lambda_{\min}(Z_t)$ grows linearly in time $t$, assuming $t$ is sufficiently large.

\begin{lemma}[implicit in \citet{kannan2018smoothed}]
Consider action-independent perturbations.
Fix any \GreedyStyle algorithm. Consider round $t \geq \tau_0$, where
    $\tau_0 =160 \frac{R^2}{\rho^2} \log \frac{2d}{\delta} \cdot \log T$.
Then for any realization of $\theta$, with probability $1-\delta$
  \[
    \lambda_{\min}(Z_t) \ge \frac{\rho^2 t}{32 \log T}.
  \]
  \label{lem:fg_big_cov}
\end{lemma}
\begin{proof}
  The claimed conclusion follows from an argument inside the proof  of
    Lemma B.1 from \citet{kannan2018smoothed},
  plugging in
  $  \lambda_0 = \frac{\rho^2}{2\log T}$.
  This argument applies for any $t\geq \tau'_0$, where
    $\tau'_0 = \max\p{32 \log \frac{2}{\delta}, 160 \frac{R^2}{\rho^2}
  \log \frac{2d}{\delta} \cdot \log T}$.
We observe that $\tau'_0=\tau_0$ since $R \ge \rho$.
\end{proof}

Rather than use Lemma~\ref{lem:fg_big_cov} directly, we extract a key portion in its proof, encapsulate it as a standalone lemma, and extend it to our model of action-correlated perturbations. Specifically, recalling that
$Z_t :=\sum_{\tau=1}^t x_\tau x_\tau\tran$,
we zero in on the expected contribution of a single round $t$.


\begin{lemma}[implicit in \citet{kannan2018smoothed} for action-independent perturbations]\label{lem:eig_increase}
Fix any \GreedyStyle algorithm, and the latent vector $\theta$. Fix round $t$. Assume $T \ge 4K$. Condition on the event that all perturbations $\eps_{a,t}$ are at most $\hat R$, denote it with $\mE$.
Fix round $t$. Then with probability at least $\nicefrac14$,
  \[
    \lambda_{\min}\p{\E{x_t\, x_t\tran \given h_{t-1}, \mE}} \ge \frac{\rho^2}{2\log T}.
  \]
\end{lemma}

The proof is this lemma is assembled from several pieces in the analysis in  \citet{kannan2018smoothed}, which extend naturally to our perturbation model.
Qualitatively, our goal is as follows: we need to argue that the
  context vector of the chosen arm at each round has sufficient variance in
  expectation (or equivalently, $\lambda_{\min}(\E{x_t x_t\tran})$ is large) for
  us to ``learn'' about all components of $\theta$. Because of the
  perturbations, each context vector independently has high variance; however,
  we need to show that this remains true even conditioning on an arm being
  chosen. Note that this conditioning should intuitively \emph{reduce} variance, since
  it selects for arms that have been perturbed in the direction of $\thetahatt$,
  all else equal. One way to view this conditioning is to consider the rewards
  of the best arm $a$ and the second-best arm $a'$: if $a'$ has a \textbf{high}
  expected reward, then the perturbation applied to $a$ must have a large
  component in the direction of $\thetahatt$ in order for $a$ to be chosen over
  $a'$, and so conditioned on the realized context vector $x_{a',t}$, we would
  expect $x_{a,t}$ to have little variance in the direction of $\thetahatt$. On
  the other hand, if $a'$ has \textbf{low} expected reward, then the
  perturbation applied to $a$ is less constrained, allowing it to have more
  variance. Our analysis will argue that the latter case is sufficiently common:
  we'll define $\hcat$ to be the expected reward of the second-best arm $a'$,
  and argue that $\hcat$ is ``low'' with constant probability, and as a result,
$x_{a,t}$ has high variance in expectation.

\begin{proof}[Proof of Lemma~\ref{lem:eig_increase}]
Let $\thetahatt$ be the algorithm's estimate for $\theta$ at time $t$.
For ease of exposition, assume that
$\thetahatt = \sbr{\|\thetahatt\|_2 ~ 0 ~\dots ~ 0}\tran$.
This is w.l.o.g. because we can just rotate the space.
%
%

Fix arm $a$, and let $\varepsilon_{a,t}^*$ be a largest-size perturbation affecting this arm. Formally, let $S = S^*_a$ be a subset $a\in S\in \mF_t$ which maximizes perturbation size $\rho_{S,t}$, and let $\varepsilon_{a,t}^* = \perturb_{S,t}$. Let
    $ \mu^*_{a,t} := x_{a,t} - \eps^*_{a,t}  $
be the result of applying all base perturbations to arm $a$, except $\eps^*_{a,t}$. Further, let
$\varepsilon_{a,t}^{-*}$ be the tuple of all \emph{other} base perturbations, including those not affecting arm $a$:
\begin{align*}
\eps_{a,t}^{-*} &:= \rbr{\perturb_{S,t}:\,S\in\mF_t\setminus\cbr{S^*_a} }.
\end{align*}
%
Similarly to~\citet{kannan2018smoothed}, define the ``estimated best arm" among those not affected by $\eps^*_{a,t}$:
  \[
    \hcat = \max_{\text{arms $a' \in A_t\setminus S^*_a$}} \;\thetahatt\tran x_{a',t}.
  \]
Let us say that round $t$ is ``good'' (meaning the expected reward for any other
arm isn't too large) for arm $a$ if
\begin{align}\label{eq:chat-good}
  \hcat \le \thetahatt\tran \mu_{a,t}^* + \rho\, \sqrt{2 \log T}\; \|\thetahatt\|_2.
\end{align}
Our argument from here will take two steps:
\begin{enumerate}
  \item Given that arm $a$ was selected, $\hcat$ is good with constant
    probability
  \item Given that $\hcat$ is good, $x_{a,t}$ has sufficient variance
    (equivalently, $\lambda_{\min}(\E{x_{a,t}x_{a,t}\tran})$ is large).
\end{enumerate}
First, we argue that for each arm $a\in A_t$,
  \begin{align}\label{eq:pf:lem:eig_increase-1}
    \Pr\b{\hcat \text{ is good for } a \given a_t = a, \mE} \ge \tfrac14.
  \end{align}

To do so, we adapt the proof of Lemma 3.4 from~\cite{kannan2018smoothed}
to handle our general perturbation model.
  Let $A_a = \{t : \Pr[a_t \text{ is good}] \ge
  \frac{1}{2}\}$, and let $B_a = \{t : \Pr_{\varepsilon}[a_t = a] \ge
\frac{2}{T}\}$. Let $S_a$ be the rounds at which $a$ was chosen, i.e., $S_a =
\{t : a_t = a\}$. We'll argue that if $t \in B_a \cap S_a$, then $t \in A_a$. As
a result, $\{t \in S_a : t \notin A_a\} \subseteq \{t \in S_a : t \notin B_a\}$.
Since our goal is to upper-bound $\{t \in S_a : t \notin A_a\}$, it suffices to
upper-bound $\{t \in S_a : t \notin B_a\}$. Consider some $t \in B_a$. Then,
taking all probabilities over all of the perturbations, and denoting the right-hand side of \eqref{eq:chat-good} with $\Lambda$, we have:
\begin{align*}
  \Pr[\hcat \text{ is not good for } a \given a_t = a]
  &= \Pr[\hcat > \Lambda \given a_t = a] \\
  &= \frac{\Pr[\hcat > \Lambda \;\text{and}\; a_t = a]}{\Pr[a_t = a]} \\
  &\le \frac{\Pr[\hcat > \thetahatt\tran \mu_{a,t}^* + \rho \sqrt{2 \log T}
\|\thetahatt\|_2 \;\text{and}\; a_t = a]}{\nicefrac{2}{T}} \tag{$t \in B_a$} \\
  &= \frac{\Pr[\hcat > \Lambda \;\text{and}\; \hcat < \thetahatt\tran (\mu_{a,t}^* +
  \varepsilon_{a,t}^*)]}{\nicefrac{2}{T}} \\
  &= \frac{\Pr[\Lambda < \hcat < \thetahatt\tran (\mu_{a,t}^* +
  \varepsilon_{a,t}^*)]}{\nicefrac{2}{T}} \\
  &\le \frac{\Pr[\Lambda < \thetahatt\tran (\mu_{a,t}^* +
  \varepsilon_{a,t}^*)]}{\nicefrac{2}{T}} \\
  &= \frac{\Pr[\rho \sqrt{2 \log T} \|\thetahatt\|_2
  < \thetahatt\tran \varepsilon_{a,t}^*]}{\nicefrac{2}{T}} \\
  &\le \nicefrac{1}{2},
\end{align*}
where the last step follows from standard tail bounds on a Gaussian.
Thus, if $t \in B_a$, then $\Pr[\hcat \text{ is good for } a \given a_t = a] \ge
\frac{1}{2}$.

Finally, let $\Pbase{\cdot}$ be the probability taken over the randomness in all
base perturbations for round $t$ (including those not affecting arm $a$). Let
$C_t$ be the set of arms at round $t$ with probability at most $2/T$ of being
chosen over the randomness of the perturbation, i.e., $C_t = \{a \given
\Pbase{a_t = a} \le \nicefrac{2}{T}\}$. Then,
\begin{align*}
  \Pbase{t \notin B_a \given a_t = a}
  = \Pbase{a_t \in C_t}
  \le \sum_{a \in
  C_t} \Pbase{a_t = a}
  \le
  \tfrac{2}{T} |C_t| \le \tfrac{2K}{T} \le \nicefrac{1}{2}.
\end{align*}
  Since by assumption $T \ge 4K$, \eqref{eq:pf:lem:eig_increase-1} follows.

Second, we argue that for each arm $a\in A_t$,
\begin{align}\label{eq:pf:lem:eig_increase-2}
    \lambda_{\min}\p{\E{x_{a,t}x_{a,t}\tran \given a_t=a \;\text{and}\; \hcat \text{ is
    good}}} \ge \frac{\rho^2}{2\log T}.
\end{align}

To prove \eqref{eq:pf:lem:eig_increase-2}, we adapt the proof of Lemma 3.2
from~\cite{kannan2018smoothed} to handle action-correlated perturbations.
Fix arm $a\in A_t$. For brevity, let's use notation for the matrix $M_{a,t} = x_{a,t}\,x_{a,t}\tran$ and the event
$\mE_t = \cbr{ a_t=a \;\text{and}\; \hcat \text{ is good} } $. Let $\Ebase{\cdot}$ be the expectation over the randomness in all base perturbations for round $t$ (including those not affecting arm $a$). Then
\begin{align*}
  \lambda_{\min}\p{\Ebase{M_{a,t} \given \mE_t}}
  &= \lambda_{\min}\p{\Ebase{\Ebase{M_{a,t} \given \mE_t,\, \eps_{a,t}^{-*}} \given \mE_t}} \\
  &\ge \Ebase{\lambda_{\min}\rbr{\Ebase{M_{a,t}\given \mE_t,\, \eps_{a,t}^{-*}}} \given \mE_t}
\end{align*}
by superadditivity of the minimum eigenvalue.

Thus, it suffices to fix $\hcat$ and show a lower bound on
\begin{align*}
  &\lambda_{\min}\p{\Ebase{M_{a,t} \given \mE_t,\,\eps_{a,t}^{-*}}} \\
  &= \lambda_{\min}\p{\Ebase{M_{a,t} \given
       \eps_{a,t}^{-*} \;\text{and}\;
      {\thetahatt}\tran x_{a,t} \ge \hcat \;\text{and}\; \hcat \le {\thetahatt}\tran
  \mu_{a,t}^* + \rho \sqrt{2 \log T} \|\thetahatt\|_2}} \\
  &= \lambda_{\min}\p{\Ebase{M_{a,t} \given
      {\thetahatt}\tran \varepsilon_{a,t}^* \ge \hcat - {\thetahatt}\tran
      \mu_{a,t}^* \;\text{and}\; \hcat - {\thetahatt}\tran \mu_{a,t}^* \le \rho \sqrt{2 \log
  T} \|\thetahatt\|_2}} \\
  &= \lambda_{\min}\p{\Ebase{M_{a,t} \given
      {\thetahatt}\tran \varepsilon_{a,t}^* \ge b \;\text{and}\; b \le \rho \sqrt{2 \log
  T} \|\thetahatt\|_2}} \\
  &\ge \frac{\rho^2}{2 \log T},
\end{align*}
where $b = \hcat - {\thetahatt}\tran \mu_{a,t}^*$. The second line follows from
the third because conditioned on $\hcat$, $\varepsilon_{a,t}^{-*}$ provides no
additional information about $M_{a,t}$. The final inequality is referred to as
the ``diversity condition'' in~\cite{kannan2018smoothed}, and they prove that
truncated Gaussian noise (recall that we are conditioning on the event that all
the perturbations are component-wise bounded by $\hat R$) satisfies this
condition in Lemma 3.7 of~\cite{kannan2018smoothed}, with parameters $r = \rho
\sqrt{2 \log T}$ and $\lambda_0 = \frac{\rho^2}{2 \log T}$.

This completes the proof of \eqref{eq:pf:lem:eig_increase-2}. The lemma follows from \eqref{eq:pf:lem:eig_increase-1} and \eqref{eq:pf:lem:eig_increase-2}.
\end{proof}

Let $\fmt$ be the \FreqGreedy estimate for $\theta$ at time $t$, as defined in \eqref{eq:FG-est-defn}. We are interested in quantifying how the quality of this estimate improves over time. \citet{kannan2018smoothed} prove, essentially, that the distance between $\fmt$ and $\theta$ scales as
    $\sqrt{t}/\lambda_{\min}(Z_t)$.
\begin{lemma}[\citet{kannan2018smoothed}]
Consider any round $t$ in the execution of \FreqGreedy. Let $t_0$ be the last round of the previous batch. For any $\theta$ and any $\delta>0$, with probability $1-\delta$,
  \[
    \|\theta - \fmt\|_2 \le
    \frac{\sqrt{t_0 \cdot 2dR \log \tfrac{d}{\delta}}}{\lambda_{\min}(Z_{t_0})}.
  \]
  \label{lem:fmt_close}
\end{lemma}

\subsubsection{Some improvements}

We focus on batch covariance matrix $\ZB$ of a given batch in a \GreedyStyle algorithm. We would like to prove that $\lambda_{\min}(\ZB)$ is sufficiently large with high probability, as long as the batch size $Y$ is large enough. The analysis from \citet{kannan2018smoothed} (a version of Lemma~\ref{lem:fg_big_cov}) would apply, but only as long as the batch size is least as large as the $\tau_0$ from the statement of Lemma~\ref{lem:fg_big_cov}. We derive a more efficient version, essentially shaving off a factor of $8$.%
\footnote{Essentially, the factor of $160$ in Lemma~\ref{lem:fg_big_cov} is replaced with factor $\tfrac{8e^2}{(e-1)^2}<20.022$ in \eqref{eq:lem:min_ev_bg-Y}.}

\begin{lemma}
Fix a \GreedyStyle algorithm and any batch $B$ in the execution of this algorithm.  Fix $\delta>0$ and assume that the batch size $Y$ is at least
\begin{align}\label{eq:lem:min_ev_bg-Y}
    Y_0 :=  (\tfrac{R}{\rho})^2 \,
            \tfrac{8e^2}{(e-1)^2}\,
            \p{1 + \log \tfrac{2d}{\delta}}\, \log(T)
    + \tfrac{4e}{e-1} \log \tfrac{2}{\delta}.
\end{align}
Condition on the event that all perturbations in this batch are upper-bounded by $\hat R$, more formally:
\[ \mE_B = \{ \|\eps_{a,t}\|_\infty \leq \hat R:\;
    \text{ for all arms $a$ and all rounds $t$ in $B$} \}. \]
Further, condition on the latent vector $\theta$ and the history $h$ before batch $B$. Then
\begin{align}\label{eq:lem:min_ev_bg}
    \Pr\left[\;  \lambda_{\min}(\ZB) \ge R^2 \given \mE_B,h,\theta \right] \geq 1-\delta.
\end{align}
The probability in \eqref{eq:lem:min_ev_bg} is over the randomness in context arrivals and rewards in batch $B$.
  \label{lem:min_ev_bg}
\end{lemma}

The improvement over Lemma~\ref{lem:fg_big_cov} comes from two sources: we use a tail bound on the sum of
  geometric random variables instead of a Chernoff bound on a binomial random
  variable, and we derive a tighter application of the eigenvalue concentration
inequality of~\citet{tropp2012user}.

\begin{proof}
Let $t_0$ be the last round before batch $B$. Recalling \eqref{eq:ZB-defn}, let
\[ \WB = \sum_{t=t_0+1}^{t_0+Y} \E{x_t x_t\tran \given h_{t-1}} \]
be a similar sum over the expected per-round covariance matrices. Assume $Y\geq Y_0$

The proof proceeds in two steps: first we lower-bound $\lambda_{\min}(\ZB)$, and then we show that it implies \eqref{eq:lem:min_ev_bg}. Denoting
    $ m = R^2\,\tfrac{e}{e-1}\,(1+\log \tfrac{2d}{\delta})$,
we claim that
\begin{align}\label{eq:matrix_conc}
    \Pr\b{\lambda_{\min}(\WB) < m \given \mE_B,h} \le \tfrac{\delta}{2}.
\end{align}

To prove this, observe that $\WB$'s minimum eigenvalue increases by at least $\lambda_0
  = \rho^2/(2\log T)$ with probability at least $1/4$ each round by
  Lemma~\ref{lem:eig_increase}, where the randomness is over the history, \ie the
  sequence of (context, reward) pairs. If we want it to go up to $m$, this
  should take $4m/\lambda_0$ rounds in expectation. However, we need it to go to
  $m$ with high probability. Notice that this is dominated by the sum of
  $m/\lambda_0$ geometric random variables with parameter $\frac{1}{4}$. We'll
  use the following bound from~\citet{janson2017tail}: for $X = \sum_{i=1}^n X_i$
  where $X_i \sim \text{Geom}(p)$ and any $c \ge 1$,
  \[
    \Pr[X \ge c\E{X}] \le \exp\p{-n(c - 1 - \log c)}.
  \]
  Because we want the minimum eigenvalue of $\WB$ to be $m$, we need $n =
  m/\lambda_0$, so $\E{X} = 4m/\lambda_0$. Choose  $c =
  \p{1+\frac{\lambda_0}{m} \log \tfrac{2}{\delta}} \tfrac{e}{e-1}$. By
  Corollary~\ref{cor:e1ex},
  \begin{align*}
    c - 1 - \log c &\ge \tfrac{e-1}{e} \cdot c - 1 = \tfrac{\lambda_0}{m} \log \tfrac{2}{\delta}.
  \end{align*}
  Therefore,
  \[
    \Pr\b{X \ge c\E{X}} \le \exp\p{-n \cdot \tfrac{\lambda_0}{m} \log
    \tfrac{2}{\delta}} = \p{\tfrac{\delta}{2}}^{n \cdot \lambda_0/m} =
    \tfrac{\delta}{2}
  \]
  Thus, with probability $1 - \frac{\delta}{2}$, $\lambda_{\min}(\WB) \ge m$ as long as the batch size $Y$ is at least
  \[
    \frac{e}{e-1} \p{1 + \frac{\lambda_0}{m} \log \frac{2}{\delta}} \cdot
    \E{X} = \frac{4e}{e-1} \p{\frac{m}{\lambda_0} + \log \frac{2}{\delta}} = Y_0.
  \]
  This completes the proof of \eqref{eq:matrix_conc}.

To derive \eqref{eq:lem:min_ev_bg} from \eqref{eq:matrix_conc}, we proceed as follows. Consider the event
\[ \mE = \left\{\;  \lambda_{\min}(\ZB) \le R^2 \text{ and }
             \lambda_{\min}(\WB) \ge m \; \right\}. \]
Letting $\alpha = 1 - R^2/m$ and rewriting $R^2$ as $(1-\alpha)m$, we use a
concentration inequality from \citet[Theorem 1.1]{tropp2012user} (following
\citet[Lemma A.3]{kannan2018smoothed}) to guarantee that
  \[
    \Pr[\mE \given \mE_B,h]
        \le d \p{e^\alpha (1-\alpha)^{1-\alpha}}^{-m/R^2}.
  \]
  Then, using the fact that $x^x \ge e^{-1/e}$ for all $x > 0$, we
  have
\begin{align*}
\Pr[\mE \given \mE_B,h]
    &\le d\p{e^{1-R^2/m-1/e}}^{-m/R^2}
    = d\, e^{-(m-R^2-m/e)/R^2} \\
    &= d \exp\p{-\frac{\p{\frac{e-1}{e}}m}{R^2} + 1}
    \le \tfrac{\delta}{2},
  \end{align*}
  since $m \ge \frac{e}{e-1} R^2 \p{1+\log \frac{2d}{\delta}}$.
  Finally, observe that, omitting the conditioning on $\mE_B,h$, we have:
  \[
    \Pr\b{\lambda_{\min}(\ZB) \leq R^2 }
        \le \Pr\b{\mE} + \Pr\b{\lambda_{\min}(\WB) < m} \le
    \tfrac{\delta}{2} + \tfrac{\delta}{2} = \delta.
  \]
\end{proof}


\subsection{Reward Simulation with a Diverse Batch History}
\label{app:pf_bg:simulation}

We consider reward simulation with a batch history, in the sense of
Definition~\ref{def:simulation}. We show that a sufficiently ``diverse" batch
history suffices to simulate the reward for any given context vector. Coupled
with the results of Section~\ref{app:pf_bg:diversity}, it follows that batch history generated by a \GreedyStyle algorithm can simulate rewards as long as the batch size is sufficiently large.

Let us recap the definition of reward simulation (Definition~\ref{def:simulation}).  Let $\Rew(\cdot)$ be a randomized function that takes a context $x$ and outputs an independent random sample from $\mathcal{N}(\theta\tran x, 1)$. In other words, this is the realized reward for an action with context vector $x$.

\begin{definition}
Consider batch $B$ in the execution of a \GreedyStyle algorithm. Batch history $h_B$ can simulate $\Rew(\cdot)$ up to radius $R>0$ if there exists a function
    $g: \{\text{context vectors}\}\times \{ \text{batch histories $h_B$}\} \to \R$
such that $g(x,h_B)$ is identically distributed to $\Rew(x)$ conditional on the batch context matrix, for all $\theta$ and all context vectors $x\in \R^d$ with $\|x\|_2\leq R$.
\label{def:simulation-app}
\end{definition}

Note that we do not require the function $g$ to be efficiently computable. We do not require algorithms to compute $g$; a mere existence of such function suffices for our analysis.

The result in this subsection does not rely on the ``greedy" property. Instead, it applies to all ``batch-style" algorithms, defined as follows: time is divided in batches of $Y$ consecutive rounds each, and the action at each round $t$ only depends on the history up to the previous batch. The data diversity condition is formalized as $\{\lambda_{\min}(Z_B) \ge R^2 \}$; recall that it is a high-probability event, in a precise sense defined in Lemma~\ref{lem:min_ev_bg}. The result is stated as follows:

\begin{lemma}
Fix a batch-style algorithm and any batch $B$ in the execution of this algorithm.
Assume the batch covariance matrix $Z_B$ satisfies $\lambda_{\min}(Z_B) \ge R^2$. Then batch history $h_B$ can simulate $\Rew$ up to radius $R$.
  \label{lem:lin_sim}
\end{lemma}

\begin{proof}
Let us construct a suitable function $g$ for Definition~\ref{def:simulation-app}. Fix a context vector $x\in \R^d$ with $\|x\|_2\leq R$. Let $r_B$ be the vector of realized rewards in batch $B$, \ie
    $r_B = (r_t: \text{rounds $t$ in $B$})\in \R^Y$. Define
\begin{align}\label{eq:lem:lin_sim:defn-g}
     g(x, h_B) = w_B\tran\, r_B + \mc N\left(0,1-\|w_B\|_2^2\right),
     \text{where $w_B = X_B\, Z_B^{-1}\, x\in \R^Y$}.
\end{align}

Recall that the variance of the reward noise is $1$. (We can also handle a more general version in which the variance of the reward noise is $\sigma^2$. Then the noise variance in \eqref{eq:lem:lin_sim:defn-g} should be $\sigma^2\,(1-\|w_B\|_2^2)$, with essentially no modifications throughout the rest of the proof.)

Note that $w_B$ is well-defined: indeed, $Z_B$ is invertible since
  $\lambda_{\min}(Z_B) \ge R^2>0$.
In the rest of the proof we show that $g$ is as needed for Definition~\ref{def:simulation-app}.

  First, we will show that for any $x \in \R^d$ such that $\|x\|_2 \le R$, the
  weights $\weights \in \R^t$ as defined above satisfy $X_B\tran \weights = x$
  and $\|\weights\|_2 \le 1$.
  Then, we'll show that if each $r_\tau \sim \mc N(\theta \tran x_\tau, 1)$,
  then $\vrt\tran \weights + \mc N(0, 1 - \|\weights\|_2^2)
  \sim \mc N(\theta\tran x, 1)$.

  Trivially, we have
  \[
    X_B\tran \weights = X_B\tran X_B (X_B\tran X_B)^{-1} x = x
  \]
  as desired. We must now show that $\|\weights\|^2_2 \le 1$. Note that
  \[
    \|\weights\|_2^2 = \weights\tran \weights = \weights\tran X_B Z_B^{-1} x =
    x\tran Z_B^{-1} x = \|x\|_{Z_B^{-1}}^2
  \]
  where $\|v\|_M^2$ simply denotes $v\tran M v$. Thus, it is sufficient to show
  that $\|x\|_{Z_B^{-1}}^2 \le 1$. Since
  $\|x\|_2 \le R$ and $\lambda_{\min}\p{Z_B} \ge R^2$, we have by
  Lemma~\ref{lem:norm_eigen}
  \[
    Z_B \succeq R^2 I \succeq xx\tran.
  \]
  By Lemma~\ref{lem:conj_succ}, we have
  \[
    I \succeq Z_B^{-1/2} xx\tran Z_B^{-1/2}.
  \]
  Let $z = Z_B^{-1/2} x$, so $I \succeq zz\tran$. Again by
  Lemma~\ref{lem:norm_eigen}, $\lambda_{\max}(zz\tran) = z\tran z$. This means
  that
  \[
    1 \ge z\tran z = (Z_B^{-1/2}x)\tran Z_B^{-1/2} x = x\tran Z_B^{-1} x =
    \|x\|_{Z_B^{-1}}^2 = \|\weights\|_2^2
  \]
  as desired.
  Finally, observe that
  \[
    \vrt\tran \weights = (X_B\theta + \eta)\tran \weights = \theta\tran X_B\tran
    \weights + \eta \tran \weights = \theta\tran x + \eta\tran \weights
  \]
  where $\eta \sim \mc N(0, I)$ is the noise vector. Notice that
  $\eta\tran \weights \sim \mc N(0, \|\weights\|_2)$, and therefore,
  $\eta\tran \weights + \mc N(0, 1-\|\weights\|_2^2) \sim \mc N(0,
  1)$. Putting this all together, we have
  \[
    \vrt\tran \weights + \mc N(0, 1-\|\weights\|_2^2) \sim \mc
    N(\theta\tran x, 1)
  \]
  and therefore $D$ can simulate $E$ for any $x$ up to radius $R$.
\end{proof}

\subsection{Regret Bounds for \BayesGreedy}
\label{sec:bg-proofs-bg}

We apply the tools from Sections~\ref{app:pf_bg:diversity} and~\ref{app:pf_bg:simulation} to derive regret bounds for \bg. On a high level, we prove that the history collected by \bg suffices to simulate a ``slowed-down" run of any other algorithm $\ALG_0$. Therefore, when it comes to choosing the next action, \bg has at least as much information as $\ALG_0$, so its Bayesian-greedy choice cannot be worse than the choice made by $\ALG_0$.

Our analysis extends to a more general scenario which is useful for the analysis of \fg. We formulate and prove our results for this scenario directly. We consider an extended bandit model which separates data collection and reward collection. Each round $t$ proceeds as follows: the algorithm observes available actions and the context vectors for these actions, then it chooses \emph{two} actions, $a_t$ and $a'_t$, and observes the reward for the former but not the latter. We refer to $a'_t$ as the ``prediction" at round $t$. We will refer to an algorithm in this model as a bandit algorithm (which chooses actions $a_t$) with ``prediction rule" that chooses the predictions $a'_t$. More specifically, we will be interested in an arbitrary \GreedyStyle algorithm with prediction rule given by \bg, as per
\eqref{eq:BG-est-defn} on \pageref{eq:BG-est-defn}. We assume this prediction rule henceforth. We are interested in \emph{prediction regret}: a version of regret \eqref{eq:regret-def} if actions $a_t$ are replaced with predictions $a'_t$:
\begin{align}\label{eq:pred-regret-def}
\PReg(T) = \textstyle
    \sum_{t=1}^T \theta\tran x_t^* -
\theta\tran x_{a'_t, t}
\end{align}
 where $x^*_{t}$ is the context vector of the best action at round $t$, as in \eqref{eq:regret-def}.
More precisely, we are interested in \emph{Bayesian prediction regret}, the expectation of \eqref{eq:pred-regret-def} over everything: the context vectors, the rewards, the algorithm's random seed, and and the prior over $\theta$.

Thus, the main theorem of this subsection is formulated as follows:

\begin{theorem}
  Consider perturbed context generation. Let $\ALG$ be an arbitrary \GreedyStyle
  algorithm whose batch size is at least $Y_0$ from \eqref{eq:lem:min_ev_bg-Y}.
Fix any bandit algorithm $\ALG_0$, and let
    $\rReg_0(T)$
be regret of this algorithm on a particular problem
instance $\mc I$. Then on the same instance, $\ALG$ has Bayesian prediction regret
\begin{align}\label{eq:thm:bg}
  \E{\rPReg(T)} \leq Y \cdot \E{\rReg_0(T/Y)} + \tilde O(1/T).
\end{align}
\label{thm:bg}
\end{theorem}

\begin{proof}[Proof sketch] We use a $t$-round history of \ALG to simulate a $(t/Y)$-round history of $\ALG_0$. More specifically, we use each batch in the history of \ALG to simulate one round of $\ALG_0$. We prove that the simulated history of $\ALG_0$ has exactly the same distribution as the actual history, for any $\theta$. Since $\ALG$ predicts the Bayesian-optimal action  given the history (up to the previous batch), this action is at least as good (in expectation over the prior) as the one chosen by $\ALG_0$ after $t/Y$ rounds. 
\end{proof}

As a corollary, we obtain regret bounds for \bg in Theorems~\ref{thm:main-greedy} and~\ref{thm:main-worst-case}. We take $\ALG$ to be \bg. For Theorem~\ref{thm:main-worst-case}(b), we take $\ALG_0$ to be LinUCB. Thus:

\begin{corollary}
In the setting of Theorem~\ref{thm:bg}, \bg has Bayesian regret at most
    $Y \cdot \E{R_0(T/Y)} + \tilde O(1/T)$
on problem instance $\mc I$. Further, under the assumptions of Theorem~\ref{thm:main-worst-case}, \bg has Bayesian regret at most
    $\tilde O(d^2\,K^{2/3}\;T^{1/3}/\rho^2)$
on all instances.
\end{corollary}

We also obtain a similar regret bound on the Bayesian prediction regret of \fg, which is essential for Section~\ref{sec:bg-proofs-fg}.

\begin{corollary}\label{cor:thm-bg-fg}
In the setting of Theorem~\ref{thm:bg}, \fg has Bayesian prediction regret \eqref{eq:thm:bg}.
\end{corollary}


In the remainder of this subsection, we prove Theorem~\ref{thm:bg}.

Throughout the proof, we condition on the event that all perturbations are bounded by $\hat{R}$, more precisely, on the event
\begin{align}\label{eq:thm:bg-pf-E1}
\mE_1 = \left\{ \|\eps_{a,t}\|_\infty \leq \hat R:\;
    \text{ for all arms $a$ and all rounds $t$ } \right\}.
\end{align}
Recall that $\mE_1$ is a high-probability event, by \eqref{eq:bg-proofs-highprob-R-hat}.
We also condition on the event
\[
  \mE_2 = \left\{ \lambda_{\min}(Z_B) \ge R^2
  : \; \text{for each batch $B$},\right\}
\]
where $Z_B$ is the batch covariance matrix, as usual. Conditioned on $\mE_1$, this too is a high-probability event (this follows by Lemma~\ref{lem:min_ev_bg}, plugging in $\delta/T$ and taking a union bound over
all batches).

We will prove that \ALG satisfies
\begin{align}\label{eq:bg-cond}
\E{\rPReg(T) \given \mE_1, \mE_2}
    \leq Y \cdot \E{\rReg_0(\Cel{T/Y}) \given \mE_1, \mE_2},
\end{align}
where the expectation is taken over everything: the context vectors, the rewards, the algorithm's random seed, and the prior over $\theta$. Then we take care of the ``failure event"
    $\overline{\mE_1 \cap \mE_2}$.

Before we prove \eqref{eq:bg-cond}, let us argue about using the history of $\ALG$ to simulate a (shorter) run of $\ALG_0$. Fix round $t$. We use a $t$-round history of $\ALG$ to simulate a $\ty$-round run of $\ALG_0$, where $Y$ is the batch size in $\ALG$. Stating this formally requires some notation.  Let $A_t$ be the set of actions available in round $t$, and let
    $\con_t = (x_{a,t}:\, a\in A_t)$
be the corresponding tuple of contexts. Let $\CON$ be the set of all possible context tuples, more precisely, the set of all finite subsets of $\R^d$. Let $h_t$ and $h^0_t$ denote, resp., the $t$-round history of $\ALG$ and $\ALG_0$. Let $\mH_t$  denote the set of all possible $t$-round histories. Note that $h_t$ and $h^0_t$ are random variables which take values on $\mH_t$.  We want to use history $h_t$ to simulate history $h^0_\ty$. Thus, the simulation result is stated as follows:

\begin{lemma}\label{lm:bg-simulation}
Fix round $t$ and let $\sigma = (\con_1 \LDOTS \con_\ty)$ be the sequence of context arrivals up to and including round $\ty$. Then there exists a ``simulation function"
    \[ \simF = \simF_t: \mH_t\times \CON_{\ty} \to \mH_{\ty} \]
such that the simulated history $\simF(h_t,\sigma)$ is distributed identically to $h^0_{\ty}$, conditional on sequence $\sigma$, latent vector $\theta$, and events $\mE_1,\mE_2$.
\end{lemma}

\begin{proof}
Throughout this proof, condition on events $\mE_1$ and $\mE_2$. Generically,
$\simF(h_t,\sigma)$ outputs a sequence of pairs
    $\{(x_\tau, r_\tau)\}_{\tau=1}^{\lfloor t/Y \rfloor}$,
where $x_\tau$ is a context vector and $r_\tau$ is a simulated reward for this
context vector. We define $\simF(h_t,\sigma)$ by induction on $\tau$ with base case $\tau=0$. Throughout, we maintain a run of algorithm $\ALG_0$. For each step $\tau\geq 1$, suppose $\ALG_0$ is simulated up to round $\tau-1$, and the corresponding history is recorded as
    $((x_1,r_1) \LDOTS (x_{\tau-1},r_{\tau-1}))$.
Simulate the next round in the execution of $\ALG_0$ by presenting it with the action set $A_\tau$ and the corresponding context tuple $\con_\tau$. Let $x_\tau$ be the context vector chosen by $\ALG_0$. The corresponding reward $r_\tau$ is constructed using the $\tau$-th batch in $h_t$, denote it with $B$. By Lemmas~\ref{lem:min_ev_bg} and~\ref{lem:lin_sim}, the batch history
$h_B$ can simulate a single reward, in the sense of
Definition~\ref{def:simulation-app}. In particular, there exists a function
$g(x,h_B)$ with the required properties (recall that it is explicitly defined in
\eqref{eq:lem:lin_sim:defn-g}). Thus, we define $r_\tau = g(x_\tau,h_B)$, and
return $r_\tau$ as a reward to $\ALG_0$. This completes the construction of
$\simF(h_t,\sigma)$. The distribution property of $\simF(h_t,\sigma)$ is immediate from the construction.
\end{proof}

\begin{proofof}[of Equation~\eqref{eq:bg-cond}]
We argue for each batch separately, and then aggregate over all batches in the very end. Fix batch $B$, and let $t_0 = t_0(B)$ be the last round in this batch. Let $\tau = 1+t_0/Y$, and consider the context vector $x^0_\tau$ chosen by $\ALG_0$ in round $\tau$. This context vector is a randomized function $f$ of the current context tuple $\con_\tau$ and the history $h^0_{\tau-1}$:
    \[ x^0_\tau = f(\con_\tau; h^0_{\tau-1}).\]
By Lemma~\ref{lm:bg-simulation}, letting
     $\sigma = (\con_1 \LDOTS \con_\ty)$,
it holds that
\begin{align}\label{eq:bg-proof-tau}
 \E{ x^0_\tau \cdot \theta \given \sigma,\theta,\mE_1,\mE_2}
    = \E{ f(\con_\tau;\,\simF(h_{t_0},\sigma)) \cdot \theta \given \sigma,\theta,\mE_1,\mE_2}
\end{align}

Let $t$ be some round in the next batch after $B$, and let
    $x'_t = x_{a'_t,t}$,
be the context vector predicted by $\ALG$ in round $t$. Recall that $x'_t$ is a Bayesian-greedy choice from the context tuple $\con_t$, based on history $h_{t_0}$.
Observe that the Bayesian-greedy action choice from a given context tuple based on history $h_{t_0}$ cannot be worse, in terms of the Bayesian-expected reward, than any other choice from the same context tuple and based on the same history. Using \eqref{eq:bg-proof-tau}, we obtain:
\begin{align}\label{eq:bg-proof-MII-cond}
 \E{ x'_t \cdot \theta \given \con_t = \con,\mE_1,\mE_2  }
    \geq \E{ x^0_\tau \cdot \theta  \given \con_\tau = \con,\mE_1,\mE_2},
 \end{align}
for any given context tuple $\con\in\CON$ that has a non-zero arrival probability given $\mE_1 \cap\mE_2$.

Observe that $\con_t$ and $\con_\tau$ have the same distribution, even conditioned on event $\mE_1 \cap\mE_2$. (This is because the definitions of $\mE_1$ and $\mE_2$ treat all rounds in the same batch in exactly the same way.)
Therefore, we can integrate \eqref{eq:bg-proof-MII-cond} over the context tuples $\con$:
\begin{align}\label{eq:bg-proof-MII}
 \E{ x'_t \cdot \theta \given \mE_1,\mE_2  }
    \geq \E{ x^0_\tau \cdot \theta  \given \mE_1,\mE_2},
 \end{align}
Now, let us sum up \eqref{eq:bg-proof-MII} over all rounds $t$ in the next batch after $B$, denote it $\term{next}(B)$.
\begin{align}\label{eq:bg-proof-MII-B}
 \sum_{t\in \term{next}(B)} \E{ x'_t \cdot \theta \given \mE_1,\mE_2  }
    \geq Y\cdot \E{ x^0_\tau \cdot \theta  \given \mE_1,\mE_2}.
 \end{align}
Note that the right-hand side of \eqref{eq:bg-proof-MII} stays the same for all $t$, hence the factor of $Y$ on the right-hand side of \eqref{eq:bg-proof-MII-B}. This completes our analysis of a single batch $B$.

We obtain~\eqref{eq:bg-cond} by integrating over all batches $B$. Here it is essential that the expectation
    $\E{\theta\tran x_t^*}$
does not depend on round $t$, and therefore the ``regret benchmark" $\theta\tran x_t^*$ cancels out from~\eqref{eq:bg-cond}. In particular, it is essential that the context tuples $\con_t$ are identically distributed across rounds.
\end{proofof}

\begin{proofof}[of Theorem~\ref{thm:bg} given Equation~\eqref{eq:bg-cond}]
We must take care of the low-probability failure events $\overline{\mE}_1$ and
$\overline{\mE}_2$.
Specifically, we need to upper-bound the expression
  \[
    \Exp_{\theta \sim P} \b{\bpreg{T} \given \overline{\mc E}_1 \cup \overline{\mc
    E}_2} \cdot \Pr[\overline{\mc E}_1 \cup \overline{\mc E}_2].
  \]
We know that $\Pr[\overline{\mc E}_1 \cup \overline{\mc E}_2] \le \delta +
  \delta_R$.
  Lemma~\ref{lem:exp_reg_ub_er} with $\ell = \hat R$
  gives us that the instantaneous regret of every round is at most
  \begin{align*}
    2\Exp_{\theta \sim (\prior \given h_{t-1})} & \b{\|\theta\|_2\p{1 + \rho(2 +
      \sqrt{2 \log K}) + \hat R}} \\
    &\le 2\b{\p{\|\pmt\|_2 + \sqrt{d\lambda_{\max}(\pvt)}}\p{1 + \rho(2 +
      \sqrt{2 \log K}) + \hat R}}
  \end{align*}
  by Lemma~\ref{lem:gaus_norm}. Letting $\delta = \delta_R = \frac{1}{T^2}$, we
  verify that our definition of $Y$ means that Lemma~\ref{lem:min_ev_bg} indeed
  holds with probability at least $1-T^{-2}$.
  Using~\eqref{eq:bg-cond}, the Bayesian prediction regret of $\ALG$ is
  \begin{align*}
    \Exp_{\theta \sim \prior} &\b{\bpreg{T}} \\
    &\le Y \Exp_{\theta \sim \prior} \b{\basereg{\tfrac{T}{Y}}} \\
     & \qquad+ 2\,T(\delta + \delta_R)\b{\p{\|\pmt\|_2 + \sqrt{d\lambda_{\max}(\pvt)}}\p{1
      + \rho(2 + \sqrt{2 \log K}) + \hat R}} \\
    &\le Y \Exp_{\theta \sim \prior} \b{\basereg{\tfrac{T}{Y}}} + \tilde
    O\p{\tfrac{1}{T}}.
  \end{align*}
This completes the proof of Theorem~\ref{thm:bg}.
\end{proofof}

\subsection{Regret Bounds for \FreqGreedy}
\label{sec:bg-proofs-fg}

To analyze \fg, we show that its Bayesian regret is not too different from its Bayesian prediction regret, and use Corollary~\ref{cor:thm-bg-fg} to bound the latter.

\begin{theorem}
Consder perturbed context generation. Suppose prior $\prior$ is a multivariate Gaussian distribution with invertible covariance matrix $\pvt$, and the eigenvalues of $\pvt$ are at least $\rho^{-4}/T$. Then
\fg satisfies
  \[  \left|\; \E{\rReg(T) - \rPReg(T)} \; \right| \leq
    \tilde O\p{\frac{\sqrt{d}}{\rho^2}} \p{\sqrt{\lambda_{\max}(\pvt)} +
    \frac{1}{\sqrt{\lambda_{\min}(\pvt)}}},
  \]
  where $\pvt$ is the covariance matrix of the prior and $\rho$ is the perturbation size.
  \label{thm:bg_fg}
\end{theorem}

Using Corollary~\ref{cor:thm-bg-fg}, we obtain regret bounds for \fg in Theorem~\ref{thm:main-greedy} and Theorem~\ref{thm:main-worst-case}.

The remainder of this section is dedicated to proving Theorem~\ref{thm:bg_fg}. On a high level, the idea is as follows. As in the proof of Theorem~\ref{thm:bg}, we condition on the high-probability event \eqref{eq:thm:bg-pf-E1} that perturbations are bounded. We prove that
\begin{align}\label{eq:thm:bg_fg-cond}
  \left|\; \E{\rReg(T) - \rPReg(T) \given \mE_1} \; \right| \leq
    \tilde O\p{\frac{\sqrt{d}}{\rho^2}} \p{\sqrt{\lambda_{\max}(\pvt)} +
    \frac{1}{\sqrt{\lambda_{\min}(\pvt)}}}.
\end{align}
To prove this statement, we fix round $t$ and compare the action $a_t$ taken by \fg and the predicted action $a'_t$. We observe that the difference in rewards between these two actions can be upper-bounded in terms of $\bmt-\fmt$,
the difference in the $\theta$ estimates with and without knowledge of the prior. (Recall \eqref{eq:BG-est-defn} and \eqref{eq:FG-est-defn} for definitions.)
Specifically, we show that
\begin{equation}
\label{eq:inst_bound_diff}
  \E{
    \theta\tran (x_{a_t, t} - x_{a'_t, t}) \given \mE_1 }
    \le 2R\Exp_{\theta \sim \prior}\b{\|\bmt - \fmt\|_2}.
\end{equation}
The crux of the proof is to show that the difference $\|\bmt - \fmt\|_2$ is small, namely
\begin{equation}
  \label{eq:norm_bound_1_t}
  \E{\|\bmt-\fmt\|_2 \given \mE_1} = \tilde O(1/t),
\end{equation}
ignoring other parameters. Interestingly, the two estimates are much closer to each other than they are to $\theta$, as either estimate is typically $\Omega(1/\sqrt{t})$ away from $\theta$.

Thus, summing over all rounds, we get
\[ \E{\rReg(T) - \rPReg(T) \given
\mE_1} \le O(\log T) = \tilde O(1). \]


Once we prove that \eqref{eq:thm:bg_fg-cond} holds under event \eqref{eq:thm:bg-pf-E1}, the proof of Theorem~\ref{thm:bg_fg} is easily completed as follows. Recall that event \eqref{eq:thm:bg-pf-E1} happens with probability at least $1-\delta_R$. When this event fails to hold, the total regret is at most
  \begin{align*}
    2\b{\p{\|\pmt\|_2 + \sqrt{d\lambda_{\max}(\pvt)}}\p{1 + \rho(2 +
      \sqrt{2 \log K}) + \hat R}}
  \end{align*}
  by Lemma~\ref{lem:exp_reg_ub_er} (with $\ell = \hat R$)
  and Lemma~\ref{lem:gaus_norm}. Since $\delta_R = T^{-2}$, the contribution of
  regret when the high-probability bound fails is $\tilde O(1/T) \le \tilde
  O(1)$.

\subsubsection{Proof of Eq.~\eqref{eq:thm:bg_fg-cond}}

Let $\regi{t}$ and $\bpregi{t}$ be, resp., instantaneous regret and instantaneous prediction regret at time $t$. Then
  \begin{equation}
    \Exp_{\theta \sim \prior}\b{\rReg(T) - \rPReg(T)}
    = \sum_t \Exp_{\theta \sim
    \prior} \b{\regi{t} - \bpregi{t}}.
    \label{eq:reg_time}
  \end{equation}
  Thus, it suffices to bound the differences in instantaneous regret.

  Recall that at time $t$, the chosen action for \fg\ and the predicted action are, resp.,
  \begin{align*}
    \af &= \argmax_{a \in A} x_{a,t}\tran \fmt \\
    \ab &= \argmax_{a \in A} x_{a,t}\tran \bmt.
  \end{align*}
Letting $t_0 - 1 = \lfloor t/Y \rfloor$ be the last round in the previous batch,
we can formulate $\fmt$ and $\bmt$ as
\begin{align*}
    \fmt &= (\Zto)^{-1} \Xto\tran \vrto \\
    \bmt &= (\Zto + \pvt^{-1})^{-1} (\Xto\tran \vrto + \pvt^{-1} \pmt).
\end{align*}

  Therefore, we have
  \[
    \Exp_{\theta \sim \prior \given h_{t-1}}\b{\regi{t} - \bpregi{t}} = \Exp_{\theta \sim
      \prior \given h_{t-1}} \b{(x_{\ab,t} - x_{\af,t})\tran \bmt} =
      (x_{\ab,t} - x_{\af,t})\tran \bmt,
  \]
  since the mean of the posterior distribution is exactly $\bmt$, and $\bmt$ is
  deterministic given $h_{t-1}$. Taking expectation over $h_{t-1}$, we have
  \[
    \Exp_{\theta \sim \prior}\b{\regi{t} - \bpregi{t}} = \Exp_{\theta \sim \prior}
    \b{(x_{\ab,t} - x_{\af,t})\tran \bmt}.
  \]
  For any fixed $\bmt$ and $\fmt$, since \fg\ chose $\af$ over $\ab$, it must be
  the case that
  \begin{equation}
    x_{\af,t}\tran \fmt \ge x_{\ab,t}\tran \fmt.
    \label{eq:freq_choice}
  \end{equation}
  Therefore,
  \begin{align*}
    (x_{\ab,t} - x_{\af,t})\tran \bmt &= (x_{\ab,t} - x_{\af,t})\tran \fmt +
    (x_{\ab,t} - x_{\af,t})\tran (\bmt - \fmt) \\
    &\le (x_{\ab,t} - x_{\af,t})\tran (\bmt - \fmt)
    \tag{By~\eqref{eq:freq_choice}} \\
    &\le (\|x_{\ab,t}\|_2 + \|x_{\af,t}\|_2)\|\bmt - \fmt\|_2 \\
    &\le 2R\|\bmt - \fmt\|_2
  \end{align*}
Eq.~\eqref{eq:inst_bound_diff} follows.

The crux is to prove \eqref{eq:norm_bound_1_t}: to bound the expected distance between the Frequentist and Bayesian estimates for $\theta$. By expanding
  their definitions, and denoting $M = (\Zto + \pvt^{-1})^{-1}$ for succinctness,
    we have
  \begin{align*}
    \bmt - \fmt
    &= M (\Xto\tran \vrto + \pvt^{-1}
    \pmt) - \Zto^{-1} \Xto\tran \vrto \\
    &= M \b{\Xto\tran \vrto + \pvt^{-1} \pmt -
    (\Zto + \pvt^{-1})\Zto^{-1} \Xto\tran \vrto} \\
    &= M \b{\Xto\tran \vrto + \pvt^{-1} \pmt -
    \Xto\tran \vrto  - \pvt^{-1}\Zto^{-1} \Xto\tran \vrto} \\
    &= M \b{\pvt^{-1} \pmt -
    \pvt^{-1}\Zto^{-1} \Xto\tran \vrto} \\
    &= M \pvt^{-1} \p{\pmt - \fmt}.
  \end{align*}
  Next, note that
  \begin{align*}
    \|M \pvt^{-1} (\pmt - \fmt)\|_2
    &\le \|M\|_2 ~ \|\pvt^{-1} (\pmt - \fmt)\|_2 \\
    &\le \|(\Zto + \pvt)^{-1}\|_2 ~ \p{\|\pvt^{-1}(\pmt - \theta)\|_2
    + \|\pvt^{-1}\|_2 ~ \|\theta - \fmt\|_2}.
  \end{align*}
  By Lemma~\ref{lem:min_ev_sum}, $\lambda_{\min}\p{\Zto + \pvt} \ge
  \lambda_{\min}\p{\Zto}$. Therefore,
  \[
    \|(\Zto + \pvt)^{-1}\|_2  = \frac{1}{\lambda_{\min}(\Zto + \pvt)} \le
    \frac{1}{\lambda_{\min}\p{\Zto}},
  \]
  giving us
  \begin{align*}
    \|\bmt - \fmt\|_2
    &\le \frac{\|\pvt^{-1}(\pmt - \theta)\|_2 + \|\pvt^{-1}\|_2
    ~ \|\theta - \fmt\|_2}{\lambda_{\min}(\Zto)} \\
    &\le \frac{\|\pvt^{-1/2}\|_2 \|\pvt^{-1/2}(\pmt - \theta)\|_2 +
    \|\pvt^{-1}\|_2 ~ \|\theta - \fmt\|_2}{\lambda_{\min}(\Zto)} \\
    &= \frac{\p{\|\pvt^{-1/2}(\pmt - \theta)\|_2 +
        \frac{1}{\sqrt{\lambda_{\min}(\pvt)}}\|\theta -
    \fmt\|_2}}{\sqrt{\lambda_{\min}(\pvt)} \lambda_{\min}(\Zto)}.
  \end{align*}

  Next, recall that for
  \[ t_0-1 \ge t_{\min}(\delta) := 160 \tfrac{R^2}{\rho^2} \log \tfrac{2d}{\delta} \cdot \log T \]
 the following bounds hold, each with probability at least $1-\delta$:
  \begin{align*}
    \frac{1}{\lambda_{\min}\p{\Zto}} &\le \frac{32 \log T}{\rho^2
    (t_0-1)}
    \tag{Lemma~\ref{lem:fg_big_cov}} \\
    \|\theta - \fmt\|_2 &\le \frac{\sqrt{2dR (t_0-1)
    \log(d/\delta)}}{\lambda_{\min}(\Zto)} \tag{Lemma~\ref{lem:fmt_close}}
  \end{align*}
 Therefore, fixing $t_0 \geq 1+t_{\min}(\delta/2)$, with probability at least $1-\delta$ we have
  \begin{equation}
    \|\bmt - \fmt\|_2
    \le \frac{32 \log T}{\rho^2 (t_0-1) \sqrt{\lambda_{\min}(\pvt)}}
    \p{\|\pvt^{-1/2}(\pmt - \theta)\|_2 + \Phi\sqrt{d}},
    \label{eq:fg_bg1}
  \end{equation}
  where for succinctness we denote
  \begin{align*}
  \Phi := \frac{64\sqrt{R\log(2d/\delta)}\cdot\log T}{\rho^2\sqrt{(t_0-1)\lambda_{\min}(\pvt)}}.
  \end{align*}
  Note that the high-probability events we need are deterministic given
  $h_{t_0-1}$, and therefore are independent of the perturbations at time $t$.
  This means that Lemma~\ref{lem:exp_reg_ub_er} applies, with $\ell = 0$: conditioned on
  any $h_{t_0-1}$, the expected regret for round $t$ is upper-bounded by
  $2\|\theta\|_2 (1 + \rho(1+\sqrt{2\log K}))$. In particular, this holds for any
  $h_{t_0-1}$ not satisfying the high probability events from
  Lemmas~\ref{lem:fg_big_cov} and~\ref{lem:fmt_close}. Therefore, for all $t \ge
  t_{\min}(\delta)$,
  \begin{align*}
&~~    \Exp_{\theta \sim \prior} \b{\|\bmt - \fmt\|_2}\\
    &\le \Exp_{\theta \sim \prior} \Bigg[(1-\delta) \frac{32 \log T}{\rho^2
      (t_0-1) \sqrt{\lambda_{\min}(\pvt)}} \p{\|\pvt^{-1/2}(\pmt - \theta)\|_2 +
      \Phi\sqrt{d}} \\
    &\qquad\qquad+ \delta \cdot 2\|\theta\|_2 (1 + \rho(2+\sqrt{2\log K})) \Bigg] \\
    &\le \frac{32 \log T}{\rho^2 (t_0-1) \sqrt{\lambda_{\min}(\pvt)}}
    \p{\Exp_{\theta \sim \prior}\b{\|\pvt^{-1/2} (\pmt - \theta)\|_2} +
    \Phi\sqrt{d}} \\
    &\qquad+ \delta \cdot 2(\|\pmt\|_2 + \Exp_{\theta \sim \prior}\b{\|\pmt -
    \theta\|_2}) (1 + \rho(2+\sqrt{2\log K})).
  \end{align*}
  Because $\theta \sim \mc N(\pmt, \pvt)$, we have $\pvt^{-1/2}
  (\pmt - \theta) \sim \mc N(0, I)$. By Lemma~\ref{lem:gaus_norm},
  \[
    \Exp_{\theta \sim \prior} \b{\|\pvt^{-1/2} (\pmt - \theta)\|_2}
    \le \sqrt{d}
    \quad\text{and}\quad
    \Exp_{\theta \sim \prior}\b{\|\pmt - \theta\|_2} \le \sqrt{d
      \lambda_{\max}(\pvt)}.
  \]
  This means
  \begin{align*}
    \Exp_{\theta \sim \prior} \b{\|\bmt - \fmt\|_2}
    &\le \frac{32 \sqrt{d} \log T}{\rho^2 (t_0-1) \sqrt{\lambda_{\min}(\pvt)}} \p{1 +
      \Phi} \\
          &+ \delta \cdot 2(\|\pmt\|_2 + \sqrt{d \lambda_{\max}(\pvt)}) (1 +
    \rho(2+\sqrt{2\log K})).
  \end{align*}
  Since $t_0 = \Omega(t)$, for sufficiently small $\delta$, this
  proves~\eqref{eq:norm_bound_1_t}.

  We need to do a careful computation to complete the proof of Eq.~\eqref{eq:thm:bg_fg-cond}.  We know from~\eqref{eq:inst_bound_diff} that
  \begin{align*}
    \Exp_{\theta \sim \prior}\b{\rReg(T) - \rPReg(T)}
    &\le \sum_{t=1}^T 2R\Exp_{\theta \sim \prior} \b{\|\bmt - \fmt\|_2}.
  \end{align*}
  Choosing $\delta = T^{-2}$, we find that
  \[
    \sum_{t=t_{\min}(T^{-2})}^T \delta \cdot 2(\|\pmt\|_2 + \sqrt{d
    \lambda_{\max}(\pvt)}) (1 + \rho(2+\sqrt{2\log K})) = \tilde O(1),
  \]
  so this term vanishes. Furthermore,
  \begin{align*}
    \sum_{t=t_{\min}(T^{-2})}^T
    &2R\frac{32 \sqrt{d} \log T}{\rho^2 (t_0-1)
      \sqrt{\lambda_{\min}(\pvt)}} \p{1 +
      \Phi} \\
    &= \tilde O\p{\frac{R\sqrt{d}}{\rho^2\sqrt{\lambda_{\min}(\pvt)}}}
  \end{align*}
  as long as $\rho^2 \sqrt{\lambda_{\min}(\pvt)} \ge T^{-1/2}$,
  since $t_0 \ge t - Y$, and $\sum_{t=1}^T 1/t = O(\log T)$.
  Using the fact that $R = \tilde O(1)$ (since by assumption $\rho \le
  d^{-1/2}$), this is simply
  \[
    \tilde O\p{\frac{\sqrt{d}}{\rho^2\sqrt{\lambda_{\min}(\pvt)}}}.
  \]
  Finally, we note that on the first $t_{\min}(T^{-2}) = \tilde O(1/\rho^2)$
  rounds, the regret bound from Lemma~\ref{lem:exp_reg_ub_er} with $\ell = 0$
  applies, so the total regret difference is at most
  \begin{align*}
    &\Exp_{\theta \sim \prior}\b{\rReg(T) - \rPReg(T)}\\
    &\qquad\le \sum_{t=1}^{t_{\min}(T^{-2})}
    \Exp_{\theta \sim \prior}\b{\regi{t} - \bpregi{t}}
    + \sum_{t=t_{\min}(T^{-2})}^T 2R\Exp_{\theta \sim \prior} \b{\|\bmt - \fmt\|_2}, \\
    &\qquad\le t_{\min}(T^{-2}) \cdot 2(\|\pmt\|_2 + \sqrt{d
    \lambda_{\max}(\pvt)})(1 + \rho(2+\sqrt{2 \log K}))
    + \tilde O\p{\frac{\sqrt{d}}{\rho^2\sqrt{\lambda_{\min}(\pvt)}}} \\
    &\qquad= \tilde O\p{\frac{\sqrt{d \lambda_{\max}(\pvt)}}{\rho^2}}
    + \tilde O\p{\frac{\sqrt{d}}{\rho^2\sqrt{\lambda_{\min}(\pvt)}}},
  \end{align*}
which implies Eq.~\eqref{eq:thm:bg_fg-cond}.

\section{Lower Bound: Proof of Theorem~\ref{thm:LB}}
\label{sec:LB}
Here, we show that for fully-action-correlated perturbations, i.e., when every arm
at a given timestep is perturbed by the same perturbation, no algorithm can
achieve regret less than $\sqrt{T}$.

Consider the following problem instance. There are $d=2$ dimensions and $K=2$ arms,
with $\mu_{1,t} = [1 ~ 0]\tran$ and $\mu_{2,t} = [0 ~ 1]\tran$ at each round. (For intuition, one can think of them as, resp., the \emph{horizontal} arm and the \emph{vertical} arm.) There are two possible
hidden vectors: $\theta_1 = [1 + \delta ~ 1]\tran$ and $\theta_2 = [1 ~ 1 +
\delta]\tran$, occurring with probability $\tfrac12$ each. Here $\delta$ is a parameter which we specify later in the analysis. Thus, arm 1 is preferable in expectation for $\theta=\theta_1$ and arm 2
is preferable for $\theta=\theta_2$. We will show that even under perturbations, we
need $\Omega(1/\delta^2)$ samples to distinguish between them, meaning we get
$\sqrt{T}$ regret for $\delta \sim 1/\sqrt{T}$.

By definition of fully-action-correlated perturbation, at any round $t$, both  $\mu_{1,t}$ and $\mu_{2,t}$ have the same perturbation
$\eps_t \sim N(0, \rho^2I)$ added to them, where $\rho$ is a perturbation size and
$I$ is the 2-dimensional identity matrix. Given $\eps_t=\eps$, the arms'
expected rewards under $\theta_1$ and $\theta_2$ are, resp.:
\begin{align*}
  \E{\theta_1\tran (\mu_{1,t} + \eps)}
  &= 1+\delta + \theta_1 \tran \eps &
  \E{\theta_2\tran (\mu_{1,t} + \eps)}
  &= 1 + \theta_2 \tran \eps & \EqComment{arm 1},\\
  \E{\theta_1\tran (\mu_{2,t} + \eps)}
  &= 1 + \theta_1 \tran \eps &
  \E{\theta_2\tran (\mu_{2,t} + \eps)}
  &= 1+\delta + \theta_2 \tran \eps & \EqComment{arm 2}.
\end{align*}

We analyze this problem instance using a standard KL-divergence technique, \eg see \cite[Chapter 2]{laurent2000adaptive}. Compared to the standard analysis, we need to handle contexts. To this end, we fix the realized sequence of perturbation vectors $\eps_1 \LDOTS \eps_T$, and condition on the high-probability event that perturbations are not too large:
\[ \niceE = \cbr{ \|\eps_t\|_2^2 \le \Psi \text{ for all rounds $t$}}, \text{where } \Psi := 2 + 8\sqrt{\log T} + 8\log T.\]
Since each $\|\eps_t\|_2^2$ follows a $\chi^2$-distribution with 2 degrees of freedom, a standard tail bound (\eg \cite[Lemma
1]{laurent2000adaptive}) implies that
\begin{align*}\textstyle
  \Pr\sbr{\niceE}
  \ge 1 - \sum_{t=1}^T \Pr\sbr{\|\eps_t\|_2^2 > \Psi}
  \geq 1 - \nicefrac{1}{T}.
\end{align*}
Clearly, it suffices to prove a regret bound for such sequence $\eps_1 \LDOTS \eps_T$.

The rest of the analysis consists of two parts: a generic K-divergence argument leading to \eqref{eq:LB-pinskerA}, and an application of \eqref{eq:LB-pinskerA} to an execution of a given algorithm. We set
$\delta = \frac{1}{8\sqrt{T\Psi}}$.

\subsection*{A generic KL-divergence argument}

Given perturbation $\eps_t=\eps$, let $\mc D_{\theta,\eps}^{(i)}$ be the
probability distribution of rewards under hidden vector $\theta$ when choosing
arm $i$ for $\theta \in \{\theta_1, \theta_2\}$ and $i \in \{1, 2\}$.
The KL-divergence between Gaussians with variance 1 and means $\xi_1,\xi_2$ is $(\xi_1 - \xi_2)^2/2$.  Since rewards are assumed to be Gaussian with variance 1, the KL-divergence
between the reward distributions of the two arms is
\begin{align*}
  \KL\rbr{\mc D_{\theta_1,\eps}^{(i)}, \mc D_{\theta_2,\eps}^{(i)}}
  &\le \rbr{\delta + |(\theta_1 - \theta_2)\tran \eps|}^2/2 \\
  &\le 2\,\max\rbr{\delta^2,\; ((\theta_1 - \theta_2)\tran \eps}^2) \\
  &\le 2\,\max\rbr{\delta^2,\, \|\theta_1 - \theta_2\|_2^2\, \|\eps\|_2^2} \\
  &\le 2\,\max\rbr{\delta^2,\, 2\delta^2\, \|\eps\|_2^2}
  \le 4\delta^2 \Psi \qquad \EqComment{under event $\niceE$}.
\end{align*}
Let $\mc D_{\theta,\eps} = \mc D_{\theta,\eps}^{(1)} \times \mc D_{\theta,\eps}^{(2)}$ be the
joint distribution of  rewards from both arms, fixing the perturbation. By the chain rule,
\begin{align*}
  \KL\rbr{\mc D_{\theta_1,\eps}, \mc D_{\theta_2,\eps}}
  &\le 8\delta^2 \Psi.
\end{align*}

Fix a realized sequence of perturbations $\bm \eps = \{\eps_t\}_{t\in[T]}$ which satisfies $\niceE$. Let
$\Omega = (\R\times\R)^T$ be the event space of rewards from the two arms, with events of the form
    $\rbr{r_{1,t},\,r_{2,t}}_{t\in[T]}$
Note that $\theta_1$ and $\theta_2$ each impose distributions over $\Omega$, call them
    $p = \prod_{t\in[T]} p_t$ and $q = \prod_{t\in[T]} q_t$,
respectively.  By a standard application of Pinsker's inequality (\eg see Lemma 2.5 in
\cite{slivkins-MABbook}), for any event $A\subset \Omega$ it holds that
\begin{align}
  2\rbr{p(A) - q(A)}^2
  &\le \KL(p, q)
  = \textstyle \sum_{t\in[T]}\,\KL(p_t, q_t) \nonumber\\
  &\le T \cdot \KL\rbr{\mc D_{\theta_1,\eps}, \mc D_{\theta_2,\eps}} \nonumber\\
  &\le 8 \delta^2 T \Psi \nonumber \\
  |p(A) - q(A)|
  &\le 2\delta \sqrt{T\Psi}. \label{eq:LB-pinskerA}
\end{align}

\subsection*{Using \eqref{eq:LB-pinskerA} to bound regret}

Consider any deterministic algorithm $\ALG$ for linear contextual bandits. Let $A$ be the event that $\ALG$ chooses arm 1 in at least $\nicefrac{T}{2}$ rounds. Since the algorithm is deterministic, $A$ can be interpreted as an event in $\Omega$. Note that if $A$ occurs when $\theta = \theta_2$ or if $\neg A$ occurs when $\theta = \theta_1$, then $\ALG$ incurs
expected regret $\Omega(\delta T) = \tilde{\Omega}(\sqrt{T})$, as desired.

We consider two cases, depending on whether $p(A) \ge \nicefrac{1}{2}$.

\textbf{Case 1:} $p(A) \ge \nicefrac{1}{2}$. Then, $q(A) \ge \nicefrac{1}{2} - 2 \delta \sqrt{T\Psi} = \nicefrac 14$, and
expected regret is
\begin{align*}
  \E{R(T) }
  &\ge \Pr[\theta = \theta_2] \cdot \E{\reg{T} \given \theta = \theta_2}
  \\
  &\ge \nicefrac{1}{2}\cdot \Pr[A \given \theta = \theta_2]\cdot \E{\reg{T} \given \theta
  = \theta_2, A} \\
  &\ge \nicefrac{1}{2} \cdot q(A) \cdot \tfrac{\delta T}{2} \\
  &\ge \delta T/16.
\end{align*}

\textbf{Case 2:} $p(A) < \nicefrac{1}{2}$. Then, expected regret is
\begin{align*}
  \E{R(T) }
  &\ge \Pr[\theta = \theta_1] \cdot \E{\reg{T} \given \theta = \theta_1}
  \\
  &\ge \nicefrac{1}{2}\cdot \Pr[\neg A \given \theta = \theta_1]
    \E{\reg{T} \given  \theta = \theta_1, \neg A} \\
  &\ge \nicefrac{1}{2} \cdot (1-p(A)) \cdot \tfrac{\delta T}{2} \\
  &\ge \delta T/8.
\end{align*}

Thus, $\E{R(T)} \geq \tilde{\Omega}(\sqrt{T})$. This extends to randomized algorithms by taking expectations over the algorithm's random seed.


\OMIT{This lower bound implies a clear separation between the independent and
correlated perturbation cases: while regret for independent perturbations is on
the order of at most $T^{1/3}$, worst-case regret for correlated perturbations
is at least $T^{1/2}$. Our results still imply that the greedy algorithm is
optimal in both cases.}

\section{LinUCB with Perturbed Contexts}
\label{app:linucb}
We prove Theorem~\ref{thm:main-worst-case}(a), a Bayesian regret bound for the LinUCB algorithm under perturbed context generation. For this section, we focus on action-independent perturbation with perturbation size $\rho$, and posit
a multivariate Gaussian prior
    $\prior = \mc N(\pmt, \pvt)$,
with mean vector $\pmt \in \R^d$ and invertible covariate matrix $\pvt \in \R^{d\times d}$.

\subsection{Preliminaries: LinUCB algorithm}

LinUCB is a well-known algorithm for linear contextual bandits, which implements the paradigm of `optimism under uncertainty'. The idea is to evaluate each action ``optimistically"---assuming the best-case scenario for this action---and then choose an action with the best optimistic evaluation. For the basic setting of multi-armed bandits, one chooses an action with the highest upper confidence bound (henceforth, UCB) on its mean reward. The UCB is computed as the sample average of the reward for this action plus a term which captures the amount of uncertainty. (This is a seminal algorithm called UCB1 \citep{bandits-ucb1}.)

Going back linear contextual bandits, the high-level idea is to compute a confidence region $\Theta_t \subset \R^d$ in each round $t$ such that $\theta\in \Theta_t$ with high
probability, and choose an action $a$ which maximizes the optimistic
reward estimate $\sup_{\theta \in \Theta_t} x_{a,t}\tran \theta$.
Concretely, one uses regression to form an empirical estimate
$\thetahatt$ for $\theta$.  Concentration techniques lead to
high-probability bounds of the form
$|x\tran (\theta -\thetahatt)| \leq f(t) \sqrt{x\tran Z_t^{-1}x}$, where
the \emph{interval width function} $f(t)$ may depend on hyperparameters and features of
the instance. LinUCB simply chooses an action
\begin{equation}
  a_t^{LinUCB} := \argmax_a x_{a,t}\tran \thetahatt + f(t) \sqrt{x_{a,t}\tran Z_t^{-1}
  x_{a,t}}.
  \label{eq:linucb_def}
\end{equation}

We focus on a version from \citet{Csaba-nips11}, with
\begin{equation}
  f(t) = S+\sqrt{d c_0\log (T+tTL^2)},
  \label{eq:Abbasi-f}
\end{equation}
here $L$ and $S$ are known upper bounds on $\|x_{a,t}\|_2$ and $\|\theta\|_2$, respectively, and $c_0$ is a parameter. For any $c_0\geq 1$, one obtains regret
    $\tilde{O}(dS\sqrt{c_0\, T})$,
with only a $\polylog$ dependence on $TL/d$ \citep{Csaba-nips11}.


\subsection{Our result}
Recall that $\rho$ denotes perturbation size, and $\pmt = \E{\theta}$ is the prior mean of the latent vector $\theta$. The parameters from \eqref{eq:Abbasi-f} are set as follows:
\begin{align}
    L &\geq 1 + \rho \sqrt{2d \log(2T^3Kd)}, \nonumber \\
    S &\geq \|\pmt\|_2 + \sqrt{3d\log T}
    \quad \text{(and $S< T$)} \label{eq:LinUCB-params}\\
    c_0 &= 1. \nonumber
\end{align}

\begin{remark}
Ideally we would like to set $L,S$ according to \eqref{eq:LinUCB-params} with equalities. We consider a more permissive version with inequalities so as to not require the exact knowledge of $\rho$ and $\|\pmt\|_2$.
While the original result in \citet{Csaba-nips11} requires
    $\|x_{a,t}\|_2\leq L$ and $\|\theta\|_2\leq S$,
in our setting this only happens with high probability.
\end{remark}

We prove the following theorem (which implies Theorem~\ref{thm:main-worst-case}(a)):

\begin{theorem}
Assume perturbed context generation, with action-independent perturbation. Further, suppose that the maximal eigenvalue of the covariance matrix $\Sigma$ of the prior $\prior$ is at most $1$, and the mean vector satisfies
    $\|\pmt\|_2\geq 1+\sqrt{3 \log T} $.
The version of LinUCB with interval width function \eqref{eq:Abbasi-f} and parameters given by \eqref{eq:LinUCB-params} has Bayesian regret at most
\begin{align}\label{eq:thm:LunUCB-main}
    T^{1/3} \left( d^2\,S\,(K^2/\rho)^{1/3} \right)\cdot \polylog(TKLd).
\end{align}
\label{thm:LunUCB-main}
\end{theorem}

\begin{remark}
The theorem also holds if the assumption on $\|\pmt\|_2$ is replaced with
$d \ge \frac{\log T}{\log \log T}$. The only change in the analysis is that in the concluding steps (Section~\ref{app:linucb-coda}), we use Lemma~\ref{lem:smooth_oful_ex}(b) instead of Lemma~\ref{lem:smooth_oful_ex}(a).
\end{remark}

\subsection{Key steps of the analysis}

On a high level, our analysis proceeds as follows. We massage algorithm's regret so as to elucidate the dependence on the number of rounds with small ``gap" between the best and second-best action, call it $N$. This step does not rely on perturbed context generation, and makes use of the analysis from \citet{Csaba-nips11}. The crux is that we derive a much stronger upper-bound on $\E{N}$ under perturbed context generation. The analysis relies on some non-trivial technicalities on bounding the deviations from the ``high-probability" behavior, which are gathered in Section~\ref{app:linucb-deviations}.

We reuse the analysis in \citet{Csaba-nips11} via the following lemma.%
\footnote{Lemma~\ref{lem:reg_sq_bound}(a) is implicit in the proof of Theorem 3 from \citet{Csaba-nips11}, and Lemma~\ref{lem:reg_sq_bound}(b) is asserted by \citet[Lemma 10]{Csaba-nips11}.}
 To state this lemma,
define the instantaneous regret at time $t$ as
    $\iR{t} = \theta\tran x_t^* - \theta\tran x_{a_t, t} $,
and let
\[
   \beta_T = \p{\sqrt{d \log \p{T(1 + TL^2)}} + S}^2.
  \]

\begin{lemma}[\citet{Csaba-nips11}]
Consider a problem instance with reward noise $\mc N(0, 1)$ and a specific realization of latent vector $\theta$ and contexts $x_{a,t}$. Consider LinUCB with parameters $L,S,c_0$ that satisfy $\|x_{a,t}\|_2\leq L$, $\|\theta\|_2 \le S$, and $c_0= 1$. Then
\begin{OneLiners}
\item[(a)] with probability at least $1-\tfrac{1}{T}$ (over the randomness in the rewards),
\[
    \textstyle  \sum_{t=1}^T\; \iR{t}^2 \le 16 \beta_T\; \log(\det(Z_t + I)),
  \]
  where $Z_t$ is the ``empirical covariance matrix" at time $t$:
  \[ \textstyle Z_t =\sum_{\tau=1}^t x_\tau x_\tau\tran\in \R^{d\times d}. \]
\item[(b)] $\det(Z_t+I) \le (1 +tL^2/d)^d$.
\end{OneLiners}
  \label{lem:reg_sq_bound}
\end{lemma}

The following lemma captures the essence of the proof of Theorem~\ref{thm:LunUCB-main}. From here on, we assume perturbed context generation without further notice. In particular, reward noise is $\mc N(0, 1)$.

\begin{lemma}
Suppose parameter $L$ is set as in \eqref{eq:LinUCB-params}. Consider a problem instance with a specific realization of $\theta$ such that $\|\theta\|_2 \le S$.
Then,
  \begin{align*}
    \E{\creg{}{T}} \le \|\theta\|_2^{-1/3}\;\p{\frac{1}{2\sqrt{\pi}} + 16 \beta_T\, d
    \log(1 + TL^2/d)} \p{\frac{TK^2}{\rho}}^{1/3} + \tilde
    O\p{1}.
  \end{align*}
  \label{lem:smooth_oful_step}
\end{lemma}

\begin{proof}
We will prove that for any $\gamma > 0$,
  \begin{align}\label{eq:pf:lem:smooth_oful_step}
    \E{\creg{}{T}}
    &\le T \cdot \frac{\gamma^2
    K^2}{2\rho\|\theta\|_2\sqrt{\pi}} + \frac{1}{\gamma} 16 \beta_T\, d
    \log(1 + TL^2/d) + \tilde O(1).
  \end{align}
The Lemma easily follows by setting $\gamma = (TK^2/(\rho \|\theta\|_2))^{-1/3}$.

Fix some $\gamma>0$. We distinguish
between rounds $t$ with $\iR{t}<\gamma $ and those with $\iR{t}\geq \gamma$:
\begin{align}\label{eq:pf:lem:smooth_oful_step:2}
  \creg{}{T} &= \sum_{t=1}^T \iR{t}
  \le \sum_{t \in \mc T_\gamma} \iR{t} + \sum_{t=1}^T
  \frac{\iR{t}^2}{\gamma}
  \le \gamma |\mc T_\gamma| + \frac{1}{\gamma} \sum_{t=1}^T \iR{t}^2,
\end{align}
where $\mc T_\gamma = \{t : \iR{t} \in (0, \gamma)\}$.

We use Lemma~\ref{lem:reg_sq_bound} to upper-bound the second summand in \eqref{eq:pf:lem:smooth_oful_step:2}. To this end, we condition on the event that every
component of every perturbation $\varepsilon_{a, t}$ has absolute value at most $\sqrt{2 \log{2T^3
Kd}}$; denote this event by $U$. This implies $\|x_{a,t}\|_2 \le L$ for all
actions $a$ and all rounds $t$. By Lemma~\ref{lem:subg_union_bound}, $U$ is a high-probability event:
    $\Pr[U] \ge 1 - \frac{1}{T^2}$.
Now we are ready to apply Lemma~\ref{lem:reg_sq_bound}:
\begin{align}\label{eq:pf:lem:smooth_oful_step:3}
\textstyle \E{\sum_{t=1}^T \iR{t}^2 \given U}
    \leq 16\,d\,\beta_T \,\log(1+tL^2/d).
\end{align}
To plug this into \eqref{eq:pf:lem:smooth_oful_step:2}, we need to account for the low-probability event $\bar{U}$. We need to be careful because $R_t$ could, with low probability, be arbitrarily large. By Lemma~\ref{lem:exp_reg_ub_er} with $\ell = 0$,
\begin{align*}
\E{R_t \given \bar U}
    &\le 2\b{\|\theta\|_2 \p{1 + \rho(1+ \sqrt{2 \log K}) + \sqrt{2 \log(2T^3 Kd)}}} \\
\E{\creg{}{T} \given \bar U} \Pr[\bar U]
&= \textstyle \sum_{t=1}^T\;\E{R_t \given \bar U} /T^2 < \tilde O(1). \\
\E{\creg{}{T} \given U}\; \Pr[U]
    &\leq \textstyle \gamma\, \E{\;|\mc T_\gamma|\;}
        + \frac{1}{\gamma} \E{\sum_{t=1}^T \iR{t}^2 \given U}
            \qquad\EqComment{by \eqref{eq:pf:lem:smooth_oful_step:2}}
\end{align*}
Putting this together and using \eqref{eq:pf:lem:smooth_oful_step:3}, we obtain:
\begin{align}\label{eq:pf:lem:smooth_oful_step:4}
\E{\creg{}{T}}
    \leq
        \gamma\, \E{\;|\mc T_\gamma|\;}
        + \frac{16}{\gamma}\,d\,\beta_T \,\log(1+tL^2/d)+ \tilde O(1).
\end{align}

To obtain \eqref{eq:pf:lem:smooth_oful_step}, we analyze the first summand in \eqref{eq:pf:lem:smooth_oful_step:4}. Let $\Delta_t$ be the ``gap" at time $t$: the difference in expected rewards
  between the best and second-best actions at time $t$ (where ``best" and
  ``second-best" is according to expected rewards). Here, we're taking
  expectations \emph{after} the perturbations are applied, so the only
  randomness comes from the noisy rewards. Consider the set of rounds with small gap,
  $\mc G_\gamma := \{t : \Delta_t < \gamma\}$.
  Notice that $r_t \in (0, \gamma)$ implies $\Delta_t <\gamma$, so
    $|\mc T_\gamma| \le |\mc G_\gamma|$.

In what follows we prove an upper bound on $\E{|\mc G_\gamma|}$. This is the step where perturbed context generation is truly used. For any two arms $a_1$ and $a_2$, the gap between their expected
  rewards is
  \[
    \theta\tran(x_{a_1,t} - x_{a_2,t}) = \theta\tran(\mu_{a_1,t} -
    \mu_{a_2,t}) + \theta\tran(\varepsilon_{a_1,t} - \varepsilon_{a_2,t}).
  \]
  Therefore, the probability that the gap between those arms is smaller than
  $\gamma$ is
  \begin{align*}
    \Pr&\b{|\theta\tran(\mu_{a_1,t} -
    \mu_{a_2,t}) + \theta\tran(\varepsilon_{a_1,t} - \varepsilon_{a_2,t})| \le
    \gamma} \\
    &= \Pr\b{-\gamma - \theta\tran(\mu_{a_1,t} - \mu_{a_2,t}) \le
    \theta\tran(\varepsilon_{a_1,t} - \varepsilon_{a_2,t}) \le \gamma -
    \theta\tran(\mu_{a_1,t} - \mu_{a_2,t})}
  \end{align*}
  Since $\theta\tran\varepsilon_{a_1,t}$ and $\theta\tran\varepsilon_{a_2,t}$
  are both distributed as $\mc N(0, \rho^2 \|\theta\|_2^2)$, their difference is
  $\mc N(0, 2 \rho^2 \|\theta\|_2^2)$. The maximum value that the Gaussian
  measure takes is $\frac{1}{2\rho\|\theta\|_2\sqrt{\pi}}$, and the measure in
  any interval of width $2\gamma$ is therefore at most
  $\frac{\gamma}{\rho\|\theta\|_2\sqrt{\pi}}$. This gives us the bound
  \[
    \Pr\b{|\theta\tran(\mu_{a_1,t} - \mu_{a_2,t}) +
    \theta\tran(\varepsilon_{a_1,t} - \varepsilon_{a_2,t})| \le \gamma} \le
    \frac{\gamma}{\rho\|\theta\|_2\sqrt{\pi}}.
  \]
  Union-bounding over all $\binom{K}{2}$ pairs of actions, we have
\begin{align*}
\Pr[\Delta_t \le \gamma]
    &\le \Pr\b{\bigcup_{a_1, a_2 \in [K]}
          |\theta\tran(x_{a_1,t} - x_{a_2,t})| \le \gamma}
    \le \frac{K^2}{2} \frac{\gamma}{\rho\|\theta\|_2\sqrt{\pi}}.\\
\E{\;|\mc G_\gamma|\;}
    &= \sum_{t=1}^T \Pr[\Delta_t \le \gamma]
    \le T \cdot \frac{K^2}{2} \frac{\gamma}{\rho\|\theta\|_2\sqrt{\pi}}.
\end{align*}
Plugging this into \eqref{eq:pf:lem:smooth_oful_step:4}
  (recalling that
    $|\mc T_\gamma| \le |\mc G_\gamma|$)
  completes the proof.
\end{proof}

\subsection{Bounding the Deviations}
\label{app:linucb-deviations}

We make use of two results that bound deviations from the ``high-probability" behavior, one on $\|\theta\|_2$ and another on instantaneous regret. First, we prove high-probability upper and lower bounds on $\|\theta\|_2$ under the conditions in Theorem~\ref{thm:LunUCB-main}. Essentially, these bounds allow us to use Lemma~\ref{lem:smooth_oful_step}.

\begin{lemma}\label{lem:smooth_oful_ex}
Assume the latent vector $\theta$ comes from a multivariate Gaussian,
    $\theta \sim \mc N(\pmt, \pvt)$,
here the covariate matrix $\pvt$ satisfies $\lambda_{\max}(\pvt) \le 1$.
\begin{itemize}
\item[(a)] If $\|\pmt\|_2 \ge 1+\sqrt{3\log T}$,  then
  for sufficiently large $T$, with probability at least $1-\frac{2}{T}$, it holds that
  \begin{align}\label{eq:lem:smooth_oful_ex}
    \tfrac{1}{2\log T} \le \|\theta\|_2 \le \|\pmt\|_2 + \sqrt{3d \log T}.
  \end{align}
\item[(b)] Same conclusion if $d \ge \frac{\log T}{\log \log T}$.
\end{itemize}
\end{lemma}
\begin{proof}
  We consider two cases, based on whether $d \ge \log T/\log \log T$. We need both cases to prove part (a), and we obtain part (b) as an interesting by-product.
We repeatedly use
  Lemma~\ref{lem:chi_sq_conc}, a concentration inequality for $\chi^2$ random
  variables, to show concentration on the Gaussian norm.

  \textbf{Case 1:} $d \ge \log T/\log \log T$. \\
  Since the Gaussian measure is decreasing in
  distance from 0, the $\Pr\b{\|\theta\|_2 \le c} \le \Pr\b{\|\pmt - \theta\|_2
  \le c}$ for any $c$. In other words, the norm of a Gaussian is most likely to
  be small when its mean is 0. Let $X = \pvt^{-1/2} (\pmt - \theta)$. Note that
  $X$ has distribution $\mc N(0, I)$, and therefore $\|X\|_2^2$ has $\chi^2$
  distribution with $d$ degrees of freedom. We can bound this as
  \begin{align*}
    \Pr\b{\|\pmt - \theta\|_2 \le \frac{1}{2\log T}}
    &= \Pr\b{\|\pvt^{-1/2}X\|_2 \le \frac{1}{2\log T}} \\
    &\le \Pr\b{\sqrt{\lambda_{\max}(\pvt)}\|X\|_2 \le \frac{1}{2\log T}} \\
    &\le \Pr\b{\|X\|_2 \le \frac{1}{2\log T}} \\
    &= \Pr\b{\|X\|_2^2 \le \frac{1}{4(\log T)^2}} \\
    &\le \p{\frac{1}{4d(\log T)^2} e^{1-1/((4\log T)^2 d)}}^{d/2} \tag{By
      Lemma~\ref{lem:chi_sq_conc}} \\
    &\le \p{\frac{\log \log T}{(\log T)^3}}^{\log T/(2\log \log T)} \tag{$d \ge
    \log T/\log \log T$} \\
    &= \frac{T^{\log \log \log T/(2 \log \log T)}}{T^{3/2}} \\
    &\le T^{-1}
  \end{align*}
  Similarly, we can show
  \begin{align*}
    \Pr\b{\|\pmt - \theta\|_2 \ge \sqrt{d\log T}}
    &= \Pr\b{\|\pvt^{-1/2}X\|_2 \ge \sqrt{d \log T}} \\
    &\le \Pr\b{\sqrt{\lambda_{\max}(\pvt)}\|X\|_2 \ge \sqrt{d \log T}} \\
    &\le \Pr\b{\|X\|_2 \ge \sqrt{d\log T}} \\
    &= \Pr\b{\|X\|_2^2 \ge d \log T} \\
    &\le \p{\log T e^{1-\log T}}^{d/2} \tag{By Lemma~\ref{lem:chi_sq_conc}} \\
    &\le \p{\exp\p{1 + \log \log T - \log T}}^{\log T/(2\log \log T)} \tag{$d
    \ge \log T/\log \log T$} \\
    &= T^{(1 + \log \log T - \log T)/(2\log \log T)} \\
    &\le T^{-1}
  \end{align*}
  for $\log T > 1 + 3 \log \log T$.
  By the triangle inequality,
  \[
    \|\pmt\|_2 - \|\pmt - \theta\|_2 \le \|\theta\|_2 \le \|\pmt\|_2 + \|\pmt -
    \theta\|_2.
  \]
  Thus, in this case, $\frac{1}{2\log T} \le \|\theta\|_2 \le \|\pmt\|_2 +
  \sqrt{d \log T}$ with probability at least $1-2T^{-1}$.

  \textbf{Case 2:} $\|\pmt\|_2 \ge 1 + \sqrt{3 \log T}$ and $d < \log T/\log
  \log T$. \\
  For this part of the proof, we just need that $d < \log T$, which it is by
  assumption. Using the triangle inequality, if $\|\pmt\|_2$ is large, it
  suffices to show that $\|\pmt - \theta\|_2$ is small with high probability.
  Again, let $X = \pvt^{-1/2} (\pmt - \theta)$. Then,
  \begin{align*}
    \Pr\b{\|\pmt - \theta\|_2 \ge \sqrt{3 \log T}}
    &= \Pr\b{\|\pvt^{1/2} X\|_2 \ge \sqrt{3 \log T}} \\
    &\ge \Pr\b{\sqrt{\lambda_{\max}(\pvt)} \|X\|_2 \ge \sqrt{3 \log T}} \\
    &= \Pr\b{\|X\|_2 \ge \frac{\sqrt{3 \log T}}{\sqrt{\lambda_{\max}(\pvt)}}} \\
    &\ge \Pr\b{\|X\|_2 \ge \sqrt{3 \log T}} \\
    &= \Pr\b{\|X\|_2^2 \ge 3 \log T}
  \end{align*}
  By Lemma~\ref{lem:chi_sq_conc},
  \begin{align*}
    \Pr\b{\|X\|_2^2 \ge 3 \log T}
    &\le \p{\frac{3\log T}{d} e^{1-\frac{3\log T}{d}}}^{d/2} \\
    &= \p{T^{-3/d}e \frac{3\log T}{d}}^{d/2} \\
    &= T^{-1} \p{T^{-1/d}e \frac{3\log T}{d}}^{d/2} \\
    &\le T^{-1} \tag{for sufficiently large
    $T$} \end{align*}
  Because $\|\pmt\|_2 \ge 1 + \sqrt{3 \log T}$, $1 \le \|\theta\|_2 \le
  \|\pmt\|_2 + \sqrt{3 \log T}$ with probability at least $1-T^{-1}$.
\end{proof}

Next, we show how to upper-bound expected instantaneous regret in the worst case.%
\footnote{We state and prove this result in a slightly more general version which we use to support Section~\ref{sec:bayesian_greedy}. For the sake of this section, a special case of $\ell=0$ suffices.}

\begin{lemma}
Fix round $t$ and parameter $\ell>0$. For any $\theta$, conditioned on any history $h_{t-1}$ and the event that
  $\|\varepsilon_{a,t}\|_\infty \ge \ell$
for each arm $a$, the expected instantaneous regret of
  any algorithm at round $t$ is at most
  \[
    2\, \|\theta\|_2\p{1 + \rho(2 + \sqrt{2 \log K}) + \ell}.
  \]
  \label{lem:exp_reg_ub_er}
\end{lemma}
\begin{proof}
  The expected regret at round $t$ is upper-bounded by the reward difference
  between the best arm $x_t^*$ and the worst arm $x_t^\dagger$, which is
  \[
    \theta\tran (x_t^* - x_t^\dagger).
  \]
  Note that $x_t^* = \mu_t^* + \varepsilon_t^*$ and $x_t^\dagger = \mu_t^\dagger
  + \varepsilon_t^\dagger$. Then, this is
  \begin{align*}
    \theta\tran (x_t^* - x_t^\dagger) &=
    \theta\tran (\mu_t^* - \mu_t^\dagger) + \theta\tran (\varepsilon_t^* -
    \varepsilon_t^\dagger) \\
    &\le 2\|\theta\|_2 + \theta\tran (\varepsilon_t^* - \varepsilon_t^\dagger)
  \end{align*}
  since $\|\mu_{a,t}\|_2 \le 1$. Next, note that
  \[
    \theta\tran \varepsilon_t^* \le \max_a \theta\tran \varepsilon_{a,t}
  \]
  and
  \[
    \theta\tran \varepsilon_t^\dagger \ge \min_a \theta\tran \varepsilon_{a,t}.
  \]
  Since $\varepsilon_{a,t}$ has symmetry about the origin conditioned on the
  event that at least one component of one of the perturbations has absolute
  value at least $\ell$, i.e. $v$ and $-v$ have equal likelihood, $\max_a
  \theta\tran \varepsilon_{a,t}$ and $-\min_a \theta\tran \varepsilon_{a,t}$ are
  identically distributed. Let $\elt$ be the event that at least one of the
  components of one of the perturbations has absolute value at least $\ell$.
  This means for any choice $\mu_{a,t}$ for all $a$,
  \begin{align*}
    \Exp \b{\theta\tran (x_t^* - x_t^\dagger) \given \elt}
    &\le 2\|\theta\|_2 + 2 \Exp \b{\max_a \theta\tran
    \varepsilon_{a,t} \given \elt}
  \end{align*}
  where the expectation is taken over the perturbations at time $t$.

  Without loss of generality, let $(\varepsilon_{a',t})_j$ be the component such
  that $|(\varepsilon_{a',t})_j| \ge \ell$. Then, all other components have
  distribution $\mc N(0, \rho^2)$. Then,
  \begin{align*}
    &\Exp \b{\max_a \theta\tran \varepsilon_{a,t} \given
    \elt} \\
    &\qquad=\Exp \b{\max_a \theta\tran \varepsilon_{a,t} \given
    |(\varepsilon_{a',t})_j| \ge \ell} \\
    &\qquad=\Exp \b{\max(\theta\tran \varepsilon_{a',t}, \max_{a
    \ne a'} \theta\tran \varepsilon_{a ,t}) \given |(\varepsilon_{a',t})_j| \ge
    \ell} \\
    &\qquad\le\Exp \b{\max\p{|\theta_j (\varepsilon_{a',t})_j| +
    \sum_{i \ne j} \theta_i (\varepsilon_{a',t})_i, \max_{a \ne a'} \theta\tran
    \varepsilon_{a ,t}} \given |(\varepsilon_{a',t})_j| \ge \ell}
  \end{align*}
  Let $(\tilde \varepsilon_{a,t})_i = 0$ if $a = a'$ and $i = j$, and
  $(\varepsilon_{a,t})_i$ otherwise. In other words, we simply zero out the
  component $(\varepsilon_{a',t})_j$. Then, this is
  \begin{align*}
    &\Exp \b{\max\p{|\theta_j (\varepsilon_{a',t})_j| +
    \theta\tran \tilde \varepsilon_{a',t}, \max_{a \ne a'} \theta\tran
    \tilde \varepsilon_{a ,t}} \given |(\varepsilon_{a',t})_j| \ge \ell} \\
    &\le \Exp \b{\max_a \p{|\theta_j
    (\varepsilon_{a',t})_j| + \theta\tran \tilde \varepsilon_{a,t}} \given
    |(\varepsilon_{a',t})_j| \ge \ell} \\
    &= \Exp \b{|\theta_j (\varepsilon_{a',t})_j| + \max_a
    \p{\theta\tran \tilde \varepsilon_{a,t}} \given |(\varepsilon_{a',t})_j|
    \ge \ell} \\
    &= \Exp \b{|\theta_j (\varepsilon_{a',t})_j|\given
    |(\varepsilon_{a',t})_j| \ge \ell} + \Exp \b{\max_{a}
    \p{\theta\tran \tilde \varepsilon_{a,t}}} \\
    &\le \Exp \b{|\theta_j (\varepsilon_{a',t})_j|\given
    |(\varepsilon_{a',t})_j| \ge \ell} + \rho \|\theta\|_2\sqrt{2\log K}
  \end{align*}
  because by Lemma~\ref{lem:subgaussian_max},
  \[
    \Exp \b{\max_a \theta\tran \tilde \varepsilon_{a,t}} \le
    \Exp \b{\max_a \theta\tran \varepsilon_{a,t}} \le \rho
    \|\theta\|_2 \sqrt{2 \log K}
  \]
  Next, note that by symmetry and since $\theta_j \le \|\theta\|_2$,
  \[
    \Exp \b{|\theta_j (\varepsilon_{a',t})_j|\given
    |(\varepsilon_{a',t})_j| \ge \ell}
    \le \|\theta\|_2 \Exp \b{(\varepsilon_{a',t})_j\given
    (\varepsilon_{a',t})_j \ge \ell}.
  \]
  By Lemma~\ref{lem:gaus_exp_bound},
  \[
    \Exp \b{(\varepsilon_{a',t})_j\given
    (\varepsilon_{a',t})_j \ge \ell} \le \max(2\rho, \ell + \rho)
    \le 2\rho + \ell
  \]
  Putting this all together, the expected instantaneous regret is bounded by
  \[
    2\p{\|\theta\|_2\p{1 + \rho(2 + \sqrt{2 \log K}) + \ell}},
  \]
  proving the lemma.
\end{proof}

\subsection{Finishing the Proof of Theorem~\ref{thm:LunUCB-main}}
\label{app:linucb-coda}

We focus on the ``nice event" that \eqref{eq:lem:smooth_oful_ex} holds, denote it
$ \mE$ for brevity. In particular, note that it implies $\|\theta\|_2\leq S$. Lemma~\ref{lem:smooth_oful_step} guarantees that expected regret under this event,
    $\E{\creg{}{T} \given \mE}$, is upper-bounded by
the expression \eqref{eq:thm:LunUCB-main} in the theorem statement.

In what follows we use Lemma~\ref{lem:smooth_oful_ex}(a) and Lemma~\ref{lem:exp_reg_ub_er} guarantee that if $\mE$ fails, then the
corresponding contribution to expected regret is small. Indeed, Lemma~\ref{lem:exp_reg_ub_er} with $\ell = 0$ implies that
\begin{align*}
    \E{R_t \given \bar{\mE}\,}
       \leq BT\,\|\theta\|_2 \quad\text{for each round $t$},
\end{align*}
where $B= 1 + \rho(2 + \sqrt{2 \log K})$ is the ``blow-up factor". Since  \eqref{eq:lem:smooth_oful_ex} fails with probability at most $\tfrac{2}{T}$
by Lemma~\ref{lem:smooth_oful_ex}(a), we have
\begin{align*}
\E{\creg{}{T} \given \bar{\mE}\,} \;\Pr[\bar{\mE}\,]
    & \leq \tfrac{2B}{T}\; \E{ \|\theta\|_2 \given \bar{\mE}\,} \\
    & \leq \tfrac{2B}{T}\;
        \E{ \|\theta\|_2 \given \|\theta\|_2 \geq \tfrac{1}{2\log T}\,} \\
    &\leq O\p{\tfrac{B}{T}}\; \p{\|\pmt\|_2 + d\log T} \\
    &\leq O(1).
\end{align*}

The antecedent inequality follows by Lemma~\ref{lem:gaus_exp_norm_bound} with
            $\alpha =\tfrac{1}{2\log T} $,
using the assumption that $\lambda_{\max}(\Sigma)\leq 1$. The theorem follows.

\section*{Acknowledgments}
We thank Dylan Foster, Jon Kleinberg, and Aaron Roth for helpful discussions about these topics.


\bibliographystyle{siamplain}
\bibliography{bib-abbrv,bib-slivkins,bib-bandits,bib-AGT,bib-fairness,bib-other}

\newpage

\appendix

\section{Auxiliary Lemmas}
\label{app:lemmas}
Our proofs use a number of tools that are either known or easily follow from something that is known. We state these tools and provide the proofs for the sake of completeness.

\subsection{(Sub)gaussians and Concentration}

We rely on several known facts about Gaussian and subgaussian random variables. A random variable $X$ is called $\sigma$-subgaussian, for some $\sigma>0$, if $E[e^{\sigma X^2}]<\infty$. This includes variance-$\sigma^2$ Gaussian random variables as a special case.

\begin{lemma}
  If $X \sim \mc N(0, \sigma^2)$, then for any $t \ge 0$,
  \[
    \E{X \given X \ge t} \le \begin{cases}
      2 \sigma & t \le \sigma \\
      t + \frac{\sigma^2}{t} & t > \sigma
    \end{cases}
  \]
  \label{lem:gaus_exp_bound}
\end{lemma}
\begin{proof}
  We begin with
  \begin{align}\label{eq:pf:lem:gaus_exp_bound}
    \E{X \given X \ge t} = \frac{\frac{1}{\sigma\sqrt{2\pi}}\int_t^\infty x
    \exp\p{x^2/(2\sigma^2)} \dx}{\Pr\b{X \ge t}}.
  \end{align}
$X$ can be represented as $X = \sigma Y$, where $Y$ is a standard normal random variable. Using a tail bound for the latter (from \citet{cook2009upper}),
\[
    \Pr\b{X \ge t} = \Pr\b{Y \ge \frac{t}{\sigma}} \ge
    \frac{1}{\sqrt{2\pi}} \frac{t/\sigma}{(t/\sigma)^2 + 1}
    \exp\p{-\frac{t^2}{2\sigma^2}}.
  \]
  The numerator in \eqref{eq:pf:lem:gaus_exp_bound} is
  \begin{align*}
    \frac{1}{\sigma\sqrt{2\pi}}\int_t^\infty x \exp\p{x^2/(2\sigma^2)} \dx
    &= -\frac{1}{\sigma \sqrt{2\pi}} \cdot \sigma^2 e^{-x^2/(2\sigma^2)}
    \bigg|_t^\infty \cdot e^{-t^2/(2\sigma^2)} \\
    &= \frac{\sigma}{\sqrt{2\pi}} \exp\p{-\frac{t^2}{2\sigma^2}}.
  \end{align*}
  Combining, we have
  \begin{align*}
    \E{X \given X \ge t} &\le \frac{\frac{\sigma}{\sqrt{2\pi}}
    \exp\p{-\frac{t^2}{2\sigma^2}}}{\frac{1}{\sqrt{2\pi}}\frac{t/\sigma}{(t/\sigma)^2
    + 1} \exp\p{-\frac{t^2}{2\sigma^2}}}
    = \frac{\sigma^2 ((t/\sigma)^2 + 1)}{t}
    = t + \frac{\sigma^2}{t}
  \end{align*}
  For $t \le \sigma$, $\E{X \given X \ge t} \le \E{X \given X \ge \sigma} \le
  2\sigma$ by the above bound.
\end{proof}
\begin{lemma}
 Suppose $X \sim \mc N(0, \Sigma)$ is a Gaussian random vector with covariance matrix $\Sigma$. Then
  \[
    \E{\; \|X\|_2 \given \|X\|_2 > \alpha \;}
        \le d\p{\alpha+ \frac{\lambda_{\max}(\Sigma)}{\alpha}} \quad\
    \text{for any $\alpha \ge 0$}.
  \]
  \label{lem:gaus_exp_norm_bound}
\end{lemma}
\begin{proof}
  Assume without loss of generality that $\Sigma$ is diagonal, since the norm is
  rotationally invariant. Observe that
    $\|X\|_2 \given \forall i ~ X_i > \alpha$
  stochastically dominates $\|X\|_2 \given \|X\|_2 > \alpha$.
  (Geometrically, the latter conditioning shifts the probability mass away from the origin.)
  Therefore,
  \begin{align*}
    \E{\; \|X\|_2 \given \|X\|_2 > \alpha \;}
    &\le \E{\; \|X\|_2 \given \forall i ~ X_i > \alpha \;} \\
    &= \textstyle  \E{\sum_{i=1}^d X_i \given \forall i ~ X_i > \alpha}
    \le \sum_{i=1}^d \p{t+ \frac{\lambda_i(\Sigma)}{\alpha}}
  \end{align*}
  by Lemma~\ref{lem:gaus_exp_bound}, where
    $\lambda_i(\Sigma) \leq \lambda_{\max}(\Sigma)$ is the $i$th
  eigenvalue of $\Sigma$.
\end{proof}

\begin{fact}
  If $X$ is a $\sigma$-subgaussian random variable, then
  \[
    \Pr[|X-\E{X}| > t] \le 2e^{-t^2/(2\sigma^2)}.
  \]
  \label{fact:subg_def}
\end{fact}

\begin{lemma}
  If $X_1, \dots, X_n$ are independent $\sigma$-subgaussian random variables, then
  \begin{align*}
    \Pr\b{\max_i |X_i-\E{X_i}| > \sigma\sqrt{2\log\frac{2n}{\delta}}} \le \delta.
  \end{align*}
  \label{lem:subg_union_bound}
\end{lemma}
\begin{proof}
  For any $X_i$, we know from Fact~\ref{fact:subg_def} that
  \[
    \Pr\b{|X_i - \E{X_i}| > \sigma \sqrt{2 \log \frac{2n}{\delta}}}
    \le 2\exp\p{-\frac{2\sigma^2 \log \frac{2n}{\delta}}{2\sigma^2}}
    = 2\exp\p{-\log \frac{2n}{\delta}}
    = \frac{\delta}{n}.
  \]
  A union bound completes the proof.
\end{proof}
\begin{lemma}
  If $X_1, \dots, X_K$ are independent zero-mean $\sigma$-subgaussian random variables, then
  \[
    \textstyle \E{\max_i X_i} \le \sigma \sqrt{2 \log K}.
  \]
  \label{lem:subgaussian_max}
\end{lemma}
\begin{proof}
Let $X = \max X_i$. Since each $X_i$ is $\sigma$-subgaussian, it follows that
  \[
    \E{e^{\lambda X_i}} \le \exp\p{\frac{\lambda^2 \sigma^2}{2}}.
  \]
  Using Jensen's inequality, we have
  \begin{align*}
    \exp\p{\lambda\E{X}} 
    \le \E{\exp\p{\lambda X}} 
    = \E{\max_i \exp\p{\lambda X_i}} 
    \le \sum_{i} \E{\exp\p{\lambda X_i}} 
    \le K e^{\lambda^2\sigma^2/2}.
  \end{align*}
  Rearranging, we have
  \[
    \E{X} \le \frac{\log K}{\lambda} + \frac{\lambda \sigma^2}{2}.
  \]
  Setting $\lambda = \frac{\sqrt{2 \log K}}{\sigma}$, we have
  $  \E{X} \le \sigma \sqrt{2 \log K}$ as needed
\end{proof}

\begin{lemma}
  If $\theta \sim \mc N(\pmt, \pvt)$ where $\pmt\in \R^d$ and $\pvt \in \R^{d \times d}$, then
$\E{\; \|\theta - \pmt\|_2\; } \le \sqrt{d \lambda_{\max}(\pvt)}$.
  \label{lem:gaus_norm}
\end{lemma}
\begin{proof}
  From~\cite{chandrasekaran2012convex}, the expected norm of a standard normal
  $d$-dimensional Gaussian is at most $\sqrt{d}$. Using the fact that
$\pvt^{-1/2} (\theta - \pmt) \sim \mc N(0, I)$,
  we have
  \begin{align*}
    \E{\|\theta - \pmt\|_2} 
    &= \E{\|\pvt^{1/2} \pvt^{-1/2}(\theta - \pmt)\|_2} \\
    &\le \|\pvt^{1/2}\|_2 \E{\|\pvt^{-1/2}(\theta - \pmt)\|_2} 
    \le \sqrt{d\lambda_{\max}(\pvt)}.
  \end{align*}
\end{proof}

\begin{lemma}[Lemma 2.2 in \citet{dasgupta2003elementary}]
  If $X \sim \chi^2(d)$, \ie $X = \sum_{i=1}^d X_i^2$, where $X_1 \LDOTS X_d$ are independent standard Normal random variables, then
  \begin{align*}
    \Pr\b{X \le \beta d} &\le (\beta e^{1-\beta})^{d/2} & \text{for any $\beta\in (0,1)$}, \\
    \Pr\b{X \ge \beta d} &\le (\beta e^{1-\beta})^{d/2} & \text{for any $\beta>1$}.
  \end{align*}
  \label{lem:chi_sq_conc}
\end{lemma}

\begin{lemma}[Hoeffding bound]
  If $\bar{X} = \frac{1}{n} \sum_{i=1}^n X_i$, where the $X_i$'s are independent
  $\sigma$-subgaussian random variables with zero mean, then
  \begin{align*}
    \max\left(\Pr\b{\bar{X} \ge t},\;\Pr\b{\bar{X} \le -t}\right)
     \le \exp\p{-\frac{nt^2}{2\sigma^2}}
        \quad\text{for all $t>0$}, \\
  \max\left(
    \Pr\b{\overline{X} \le -\sigma
    \sqrt{\tfrac{2}{n}\log \tfrac{1}{\delta}}},\quad
    \Pr\b{\overline{X} \ge \sigma
    \sqrt{\tfrac{2}{n}\log \tfrac{1}{\delta}}}  \right) \le \delta
    \quad\text{for all $\delta>0$}.
  \end{align*}
  \label{lem:hoeffding}
\end{lemma}

\OMIT{ 
\subsection{KL-divergence}

We use some basic facts about KL-divergence. Let us recap the definition: given two distributions $P,Q$ on the same finite outcome space $\Omega$, KL-divergence from $P$ to $Q$ is
\[ \kl{P}{Q} := - \sum_{\omega \in \Omega} P(\omega) \log \tfrac{Q(\omega)}{P(\omega)} .\]

\begin{lemma}[High-probability Pinsker Inequality~\citep{T09}]
\label{lem:pinkser}
For any probability distributions $P$ and $Q$ over the same sample space and any arbitrary event $E$,
\[
P(E) + Q(\overline{E}) \ge \tfrac{1}{2}\, e^{-\kl{P}{Q}}.
\]
\end{lemma}

\begin{lemma}
\label{lem:bern_kl}
Let $P$ and $Q$ be Bernoulli distributions with means $p \in [1/2-\eps,
1/2+\eps]$ and $q \in [1/2-\eps, 1/2+\eps]$ respectively, with $\eps \le 1/4$. Then
  $\kl{P}{Q} \le \frac{7}{3}\,\eps^2$.
\end{lemma}
\begin{proof}
For any $\eps \le 1/4$,

\begin{align*}
\log \p{\frac{p(1-p)}{q(1-q)}} &\le
\log\p{\frac{1/4}{1/4-\eps^2}} \le
\log\p{\frac{1}{1-4\eps^2}}    \le
 \frac{14\, \eps^2}{3} \tag {By
    Lemma~\ref{lem:log_expansion}}
\\
  \kl{P}{Q} &= p \log \p{\frac{p}{q}} + (1-p) \log\p{\frac{1-p}{1-q}} \\
  &\le \p{\frac{1}{2} + \eps} \log \p{\frac{p(1-p)}{q(1-q)}}
    = \p{\frac{1}{2} + \eps} \frac{14 \eps^2}{3}
   \le \frac{7\eps^2}{2}.
\end{align*}
\end{proof}

} 

\subsection{Linear Algebra}

We use several facts from linear algebra. In what follows, recall that
$\lambda_{\min}(M)$ and $\lambda_{\max}(M)$ denote the minimal and the maximal eigenvalues of matrix $M$, resp.
For two matrices $A,B$, let us write $B \succeq A$ to mean that $B-A$ is positive semidefinite.

\begin{lemma}
    $\lambda_{\max}(vv\tran) = \|v\|_2^2$ ~~ for any $v \in \R^d$.
  \label{lem:norm_eigen}
\end{lemma}
\begin{proof}
  $vv\tran$ has rank one, so it has one eigenvector with nonzero eigenvalue. $v$
  is an eigenvector since $(vv\tran)v = (v\tran v) v$, and it has eigenvalue
  $v\tran v = \|v\|_2^2$. This is the only nonzero eigenvalue, so
  $\lambda_{\max}(vv\tran) = \|v\|_2^2$.
\end{proof}

\begin{lemma} \label{lem:conj_succ}
  For symmetric matrices $A$, $B$ with $B$ invertible,
  \[
    B \succeq A \Longleftrightarrow I \succeq B^{-1/2} A B^{-1/2}
  \]
\end{lemma}
\begin{proof}
  \begin{align*}
    B \succeq A &\Longleftrightarrow x\tran B x \ge x\tran A x \tag{$\forall x$} \\
    &\Longleftrightarrow x\tran(B-A) x \ge 0 \tag{$\forall x$} \\
    &\Longleftrightarrow x\tran B^{1/2} (I-B^{-1/2}AB^{-1/2}) B^{1/2} x \ge 0 \tag{$\forall x$} \\
    &\Longleftrightarrow x\tran (I-B^{-1/2}AB^{-1/2}) x \ge 0 \tag{$\forall x$} \\
    &\Longleftrightarrow I \succeq B^{-1/2}AB^{-1/2}.
  \end{align*}
\end{proof}

\begin{lemma}
  If $A \succeq 0$ and $B \succeq 0$, then
$\lambda_{\min}(A + B) \ge \lambda_{\min}(A)$.
  \label{lem:min_ev_sum}
\end{lemma}
\begin{proof}
  \begin{align*}
    \lambda_{\min}(A + B) &= \min_{\|x\|_2 = 1} x\tran (A+B) x \\
    &= \min_{\|x\|_2 = 1} x\tran A x + x\tran B x \\
    &\ge \min_{\|x\|_2 = 1} x\tran A x \tag{because $x\tran B x \ge 0$} \\
    &= \lambda_{\min}(A)
  \end{align*}
\end{proof}

\subsection{Logarithms}

We use several variants of standard inequalities about logarithms.

\begin{lemma}
    $x \ge \log(ex)$ for all $x > 0$.
  \label{lem:xex}
\end{lemma}
\begin{proof}
Equivalently, $x - \log(ex) \ge 0$ for $x > 0$. To show this,
  observe that
  \begin{enumerate}
    \item At $x=1$, this holds with equality.
    \item At $x = 1$, the derivative is
      \[
        \frac{d}{dx} x - \log(ex) \bigg|_{x=1} = 1 - \frac{1}{x} \bigg|_{x=1} =
        0.
      \]
    \item The entire function is convex for $x > 0$, since
      \[
        \frac{d^2}{dx^2} x - \log(ex) = \frac{d}{dx} 1 - \frac{1}{x} =
        \frac{1}{x^2} > 0.
      \]
  \end{enumerate}
  This proves the lemma.
\end{proof}
\begin{corollary}
  $  x - \log x \ge \frac{e-1}{e} x$.
  \label{cor:e1ex}
\end{corollary}
\begin{proof}
  Using Lemma~\ref{lem:xex} and letting $z = x/e$,
  \[
    x - \log x = \frac{e-1}{e} x + \frac{1}{e}x - \log x = \frac{e-1}{e} x + z -
    \log(ez) \ge \frac{e-1}{e} x
  \]
\end{proof}
\begin{lemma}
$\log\p{\frac{1}{1-x}} \le \frac{7x}{6}$ for any $x \in [0,1/4]$.
  \label{lem:log_expansion}
\end{lemma}
\begin{proof}
  First, we note that
  \[
    \tfrac{d}{dx} \log\p{\tfrac{1}{1-x}} = 1-x (-(1-x)^{-2}) \cdot (-1)
    = \tfrac{1}{1-x}
    = \sum_{i=0}^\infty x^i.
  \]
  Integrating both sides, we have
  \[
    \log\p{\tfrac{1}{1-x}} = C + \sum_{i=0}^\infty \frac{x^i}{i},
  \]
  for some constant $C$ that does not depend on $x$. Taking $x = 0$ yields $C = 0$.
  Therefore,
  \[
    \log\p{\frac{1}{1-x}} \le x + \frac{x^2}{2} \sum_{i=0}^\infty x^i = x +
    \frac{x^2}{2(1-x)} = x\p{1 + \frac{x}{2(1-x)}} \le \frac{7x}{6}.
  \]
\end{proof}

\end{document}